\documentclass[runningheads]{llncs}

\usepackage{graphicx}
  
\usepackage{systeme,mathtools}
\usepackage{amsthm}
\usepackage{amsmath}
\usepackage{amssymb}
\usepackage{pgf,tikz,pgfplots}
\usepackage{subcaption}
\usepackage{thmtools, thm-restate}
\usepackage{mathrsfs}
\usepackage{hyperref}

\DeclareMathOperator{\argmin}{argmin}
\DeclareMathOperator{\argmax}{argmax}

\DeclarePairedDelimiter\abs{\lvert}{\rvert}

\usepackage{prettyref}
\newcommand{\pref}{\prettyref}
\newrefformat{thm}{Theorem~\ref{#1}}
\newrefformat{lem}{Lemma~\ref{#1}}
\newrefformat{cha}{Chapter~\ref{#1}}
\newrefformat{sec}{Section~\ref{#1}}
\newrefformat{tab}{Table~\ref{#1}}
\newrefformat{fig}{Figure~\ref{#1}}
\newrefformat{alg}{Algorithm~\ref{#1}}
\newrefformat{exa}{Example~\ref{#1}}
\newrefformat{def}{Definition~\ref{#1}}
\newrefformat{li}{Line~\ref{#1}}
\newrefformat{eq}{Equation~\ref{#1}}
\newrefformat{eqn}{Equation~\ref{#1}}
\newrefformat{app}{Appendix~\ref{#1}}
\newrefformat{rmk}{Remark~\ref{#1}}

\usepackage[ruled,vlined,linesnumbered]{algorithm2e}

\usepackage{tipa}

\newcommand{\abstr}[1]{{\color{red}#1}}
\newcommand{\concr}[1]{{\color{blue}#1}}

\DeclareMathOperator*{\ScaleCols}{ScaleCols}
\DeclareMathOperator*{\PCMs}{PCMs}
\DeclareMathOperator*{\BinPCMs}{BinPCMs}
\newcommand\AbsCollection{\abstr{\mathfrak{A}}}
\newcommand\AbsDomain{\abstr{\mathcal{A}}}
\newcommand\AbsDomaini[1]{\abstr{\AbsDomain^{(#1)}}}
\newcommand\AbsElem{\abstr{A}}
\newcommand\AbsElemi[1]{\abstr{A^{(#1)}}}
\newcommand\ALW{\mathrm{AbstractLayerWise}\langle\abstr{\mathfrak{A}}, \concr{\Sigma}\rangle(\concr{N}, \nping, \AbsDomain)}
\newcommand\aalpha{\abstr{\alpha}}
\newcommand\haalpha{\widehat{\aalpha}}
\newcommand\czeta{\concr{\overline{\textrm{\textgamma}}}}
\newcommand\cgamma{\concr{\gamma}}
\newcommand\csigma{\concr{\sigma}}
\newcommand\csigmap{\concr{\sigma'}}
\newcommand\cSigma{\concr{\Sigma}}
\newcommand\csigmai[1]{\concr{\sigma^{(#1)}}}
\newcommand\csigmaip[1]{\concr{\sigma^{(#1)'}}}

\newcommand\IntDom{\mathtt{Int}}
\newcommand\OctDom{\mathtt{Oct}}
\newcommand\cWi[1]{\concr{W^{(#1)}}}
\newcommand\cWip[1]{\concr{W^{(#1)'}}}
\newcommand\cHi[1]{\concr{H^{(#1)}}}
\newcommand\meanrep[2]{{#1}_{/#2}}

\newcommand\partitioning{\mathcal{P}}
\newcommand\ping{\partitioning}
\newcommand\netpartitioning{\mathbb{P}}
\newcommand\nping{\netpartitioning}
\newcommand\lping[1]{\nping^{(#1)}}
\newcommand\prabhakar{Prabhakar~et~al.}
\newcommand\prabhakarcite{Prabhakar~et~al.~\cite{inn}}

\usepackage{ifthen}
\newcommand\target{arxiv}
\newcommand\onlyfor[3]{\ifthenelse{\equal{#1}{\target}}{#2}{#3}}

\newcommand\appendixproof[1]{
    \begin{proof}
        \onlyfor{arxiv}{
            Please see~\pref{#1} for the proof of this theorem.
        }{
            Please see the arXiv version for the proof of this theorem.
        }
    \end{proof}
}

\begin{document}
\title{Abstract Neural Networks}

\author{Matthew Sotoudeh \and Aditya V.\ Thakur}
\authorrunning{M.\ Sotoudeh and A.\ V.\ Thakur}

\institute{University of California, Davis, USA
\\ \email{\{masotoudeh,avthakur\}@ucdavis.edu}}

\maketitle              %
\begin{abstract}
    Deep Neural Networks (DNNs) are rapidly being applied to safety-critical
    domains such as drone and airplane control, motivating techniques for
    verifying the safety of their behavior. Unfortunately, DNN verification is
    NP-hard, with current algorithms slowing exponentially with the number of
    nodes in the DNN. 
    This paper introduces the notion of Abstract Neural Networks (ANNs), which 
    can be used to soundly overapproximate DNNs while using fewer nodes. 
    An ANN is like a DNN except weight matrices are replaced by values in
    a given abstract domain. We present a framework parameterized by the
    abstract domain and activation functions used in the DNN that can be used
    to construct a corresponding ANN. We present necessary and
    sufficient conditions on the DNN activation functions for the constructed
    ANN to soundly over-approximate the given DNN. Prior work on DNN
    abstraction was restricted to the interval domain and ReLU activation
    function. Our framework can be instantiated with other abstract domains
    such as octagons and polyhedra, as well as other activation functions such
    as Leaky ReLU, Sigmoid, and Hyperbolic Tangent.

    Code: \url{https://github.com/95616ARG/abstract_neural_networks}

\end{abstract}

\section{Introduction}
\label{sec:Introduction}
Deep Neural Networks (DNNs), defined formally in~\pref{sec:Preliminaries},
are loop-free computer programs organized into \emph{layers}, each of which
computes a linear combination of the layer's inputs, then applies some
\emph{non-linear activation function} to the resulting values. The activation
function used varies between networks, with popular activation functions
including ReLU, Hyperbolic Tangent, and Leaky
ReLU~\cite{Goodfellow:DeepLearning2016}. DNNs have rapidly become important in
a variety of applications, including image recognition and safety-critical
control systems, motivating research into the problem of verifying properties
about their behavior~\cite{reluplex,ai2}.

Although they lack loops, the use of non-linear activation functions introduces
\emph{exponential branching behavior} into the DNN semantics. It has
been shown that DNN verification is NP-hard~\cite{reluplex}. In particular, this
exponential behavior scales with the number of \emph{nodes} in a network.  DNNs
in practice have very large numbers of nodes, e.g., the aircraft
collision-avoidance DNN ACAS Xu~\cite{julian2016policy} has $300$ and a modern
image recognition network has tens of thousands~\cite{krizhevsky2012imagenet}.
The number of nodes in modern networks has also been growing with time as more
effective training methods have been found~\cite{radford2019language}.

One increasingly common way of addressing this problem is to compress the DNN
into a smaller proxy network which can be analyzed in its place. However, most
such approaches usually do not guarantee that properties of the proxy network
hold in the original network (they are unsound).  Recently,
\prabhakarcite{} introduced the notion of \emph{Interval Neural
Networks} (INNs), which can produce a smaller proxy network that is
\emph{guaranteed} to over-approximate the behavior of the original DNN. While
promising, soundness is only guaranteed with a particular activation function
(ReLU) and abstract domain (intervals).

In this work, we introduce \emph{Abstract Neural Networks} (ANNs), which are
like DNNs except weight matrices are replaced with values in an abstract
domain. Given a DNN and an abstract domain, we present an algorithm for
constructing a corresponding ANN with fewer nodes. The algorithm works by
merging groups of nodes in the DNN to form corresponding abstract nodes in the
ANN. We prove necessary and sufficient conditions on the activation functions
used for the constructed ANN to over-approximate the input DNN.  If these
conditions are met, the smaller ANN can be soundly analyzed in place of the
DNN.  Our formalization and theoretical results generalize those of \prabhakarcite{},
which are an instantiation of our framework for ReLU activation functions and
the interval domain. Our results also show
how to instantiate the algorithm such that sound abstraction can be achieved
with a variety of different abstract domains (including polytopes and octagons)
as well as many popular activation functions (including Hyperbolic Tangent,
Leaky ReLU, and Sigmoid).

\subsubsection{Outline}
In this paper, we aim to lay strong theoretical foundations for research into
abstracting neural networks for verification. \pref{sec:Overview} gives an
overview of our technique. \pref{sec:Preliminaries} defines preliminaries.
\pref{sec:ANNs} defines \emph{Abstract Neural Networks} (ANNs).
\pref{sec:Algorithm} presents an algorithm for constructing an ANN from a given
DNN. \pref{sec:Examples} motivates our theoretical results with a number of
examples. \pref{sec:ProveSound} proves our soundness theorem.
\pref{sec:Related} discusses related work, while \pref{sec:Conclusion}
concludes with a discussion of future work.

\section{Motivation}
\label{sec:Overview}

\begin{figure}
    \begin{subfigure}{0.32\textwidth}
        \begin{tikzpicture}
            \draw node[circle,draw=black] (x1) at (0, 0) {$x_1$};
            \draw node[circle,draw=black] (h1) at (1.5, 0.75) {$h_1$};
            \draw node[circle,draw=black] (h2) at (1.5, -0.75) {$h_2$};
            \draw node[circle,draw=black] (y1) at (3, 1) {$y_1$};
            \draw node[circle,draw=black] (y2) at (3, 0) {$y_2$};
            \draw node[circle,draw=black] (y3) at (3, -1) {$y_3$};

            \draw[->] (x1) -- (h1) node[midway,above] {$1$};
            \draw[->] (x1) -- (h2) node[midway,below] {$-1$};

            \draw[->] (h1) -- (y1) node[midway,above,xshift=1mm,yshift=1mm] {$1$};
            \draw[->] (h1) -- (y2) node[near end,above,xshift=2mm,yshift=-1mm] {$1$};
            \draw[->] (h1) -- (y3) node[near end,below] {$0$};

            \draw[->] (h2) -- (y1) node[near end,above] {$1$};
            \draw[->] (h2) -- (y2) node[near end,above] {$0$};
            \draw[->] (h2) -- (y3) node[midway,below] {$1$};
        \end{tikzpicture}
        \caption{DNN $N_1$}
        \label{fig:OverviewDNN}
    \end{subfigure}
    \hfill
    \begin{subfigure}{0.32\textwidth}
        \begin{tikzpicture}
            \draw node[circle,draw=black] (x1) at (0, 0) {$x_1$};
            \draw node[circle,draw=black] (h) at (1.5, 0) {$\overline{h}$};
            \draw node[circle,draw=black] (y1) at (3, 1) {$y_1$};
            \draw node[circle,draw=black] (y2) at (3, 0) {$y_2$};
            \draw node[circle,draw=black] (y3) at (3, -1) {$y_3$};

            \draw[->] (x1) -- (h) node[midway,above] {\scriptsize$[-1, 1]$};

            \draw[->] (h) -- (y1) node[midway,above] {\scriptsize$[2, 2]$};
            \draw[->] (h) -- (y2) node[near end,above,xshift=-1mm] {\scriptsize$[0, 2]$};
            \draw[->] (h) -- (y3) node[near end,below,xshift=-2mm] {\scriptsize$[0, 2]$};
        \end{tikzpicture}
        \caption{Corresponding INN}
        \label{fig:OverviewINN}
    \end{subfigure}
    \hfill
    \begin{subfigure}{0.32\textwidth}
        \begin{tikzpicture}
            \draw node[circle,draw=black] (x1) at (0, 0) {$x_1$};
            \draw node[circle,draw=black] (h) at (1.5, 0) {$\overline{h}$};
            \draw node[circle,draw=black] (y1) at (3, 1) {$y_1$};
            \draw node[circle,draw=black] (y2) at (3, 0) {$y_2$};
            \draw node[circle,draw=black] (y3) at (3, -1) {$y_3$};

            \draw[->] (x1) -- (h) node[midway,above] {$0.5$};

            \draw[->] (h) -- (y1) node[midway,above] {$2$};
            \draw[->] (h) -- (y2) node[midway,above] {$2$};
            \draw[->] (h) -- (y3) node[midway,above] {$0$};
        \end{tikzpicture}
        \caption{One instantiation of the INN}
        \label{fig:OverviewInst}
    \end{subfigure}
    \caption{Example DNN to INN and one of many instantiations of the INN.}
\end{figure}
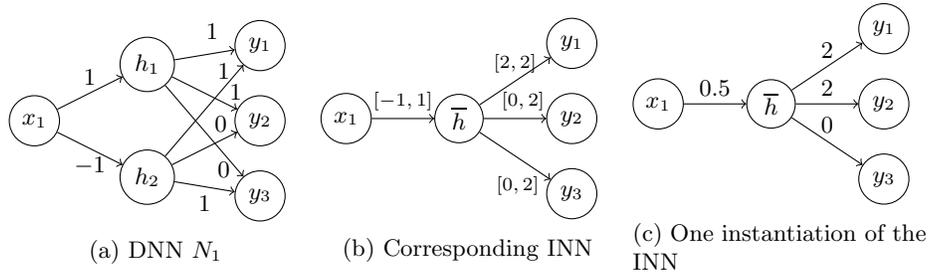

DNNs are often denoted by a graph of the form shown in~\pref{fig:OverviewDNN}.
The input node $x_1$ is assigned the \emph{input value}, then the values of
$h_1$ and $h_2$ are computed by first a linear combination of the values of the
previous layer (in this case $x_1$) followed by some \emph{non-linear
activation function.} The behavior of the network is dependent on the
non-linear activation function used. We will assume that the output layer with nodes
$y_1$, $y_2$, and $y_3$ uses the identity activation function $I(x) = x$. For
the hidden layer with nodes $h_1$ and $h_2$ we will consider two scenarios,
each using one of the following two activation functions:
\[
    \sigma(x) =
    \begin{cases}
        x &\text{if } x \geq 0 \\
        0 &\text{otherwise.}
    \end{cases}
    \qquad
    \phi(x) =
    \begin{cases}
        x &\text{if } x \geq 0 \\
        0.5x &\text{otherwise.}
    \end{cases}.
\]

Using $\sigma$ as the activation function for the hidden layer, when $x_1 = 1$
we have $h_1 = \sigma(1x_1) = 1$ and $h_2 = \sigma(-1x_1) = 0$. That in turn
gives us $y_1 = I(1h_1 + 1h_2) = 1$, $y_2 = I(1h_1 + 0h_2) = 1$, and $y_3 =
I(0h_1 + 1h_2) = 0$.

Using $\sigma$ as the activation function for the hidden layer, when $x_1 = 1$, we have
\begin{align*} 
    h_1 &= \sigma(1x_1) = 1 & h_2 &= \sigma(-1x_1) = 0 &  &\\
    y_1 &= I(1h_1 + 1h_2) = 1 & y_2 &= I(1h_1 + 0h_2) = 1 & y_3 &= I(0h_1 + 1h_2) = 0.
\end{align*}

Using $\phi$ as the activation function for the hidden layer, when $x_1 = 1$, we have
\begin{align*}
    h_1 &= \phi(1) = 1 & h_2 &= \phi(-1) = -0.5 & & \\
    y_1 &= 0.5 & y_2 &= 1 & y_3 &=-0.5.
\end{align*}

\subsection{Merging Nodes}
Our goal is to \emph{merge} nodes and their corresponding weights in this DNN
to produce a smaller network that over-approximates the behavior of the
original one. One way of doing this was proposed by~\prabhakarcite{}, where nodes
within a layer can be merged and the \emph{weighted interval hull} of their
edge weights is taken. For example, if we merge all of the $h_i$ nodes together
into a single $\overline{h}$ node, this process results in an \emph{Interval
Neural Network} (INN) shown in~\pref{fig:OverviewINN}.

Intuitively, given this new INN we can form a \emph{DNN instantiation} by
picking any weight within the interval for each edge. We can then find the
output of this DNN instantiation on, say, $x_1 = 1$. We take the \emph{output
of the INN} on an input $x_1$ to be the set of \emph{all} such $(y_1, y_2,
y_3)$ triples outputted by some such instantiated DNN on $x_1$.

For example, we can take the instantiation in~\pref{fig:OverviewInst}.  Using
the $\sigma$ activation function, this implies $(y_1 = 1, y_2 = 1, y_3 = 0)$ is
in the output set of the INN on input $x_1 = 1$. In fact, the results
of~\prabhakarcite{} show that, if the $\sigma$ activation function is used, then for
\emph{any} input $x_1$ we will have some assignment to the weights which
produces the same output as the original DNN (although many assignments will
produce different outputs --- the output set is an \emph{over-approximation} of
the behavior of the original network).

However, something different happens if the network were using the $\phi$
activation function, a case that was not considered by~\prabhakarcite{}. In that
scenario, the original DNN had an output of $(0.5, 1, -0.5)$, so if the INN
were to soundly over-approximate it there would need to be some instantiation
of the weights where $y_1$ and $y_3$ could have opposite signs. But this cannot
happen --- both will have the same (or zero) sign as $\overline{h}$!

These examples highlight the fact that the soundness of the algorithm
from~\prabhakarcite{} is specific to the ReLU activation function ($\sigma$ above)
and Interval abstract domain. Their results make no statement about whether
INNs over-approximate DNNs using different activation functions (such as $\phi$
above), or if abstractions using different domains (such as the \emph{Octagon}
Neural Networks defined in~\pref{def:ONNs}) also permit sound DNN
over-approximation.

This paper develops a general framework for such DNN abstractions,
parameterized by the abstract domain and activation functions used. In this
framework, we prove \emph{necessary and sufficient} conditions on the
activation functions for a \emph{Layer-Wise Abstraction Algorithm} generalizing
that of~\prabhakarcite{} to produce an ANN soundly over-approximating the given DNN.
Finally, we discuss ways to modify the abstraction algorithm in order to
soundly over-approximate common DNN architectures that fail the necessary
conditions, extending the applicability of model abstraction to almost all
currently-used DNNs.

These results lay a solid theoretical foundation for research into Abstract
Neural Networks. Because our algorithm and proofs are parameterized by the
abstract domain and activation functions used, our proofs allow practitioners
to experiment with different abstractions, activation functions, and
optimizations without having to re-prove soundness for their particular
instantiation (which, as we will see in~\pref{sec:ProveSound}, is a
surprisingly subtle process).

\section{Preliminaries}
\label{sec:Preliminaries}

In this section we define Deep Neural Networks and a number of commonly-used
activation functions.

\subsection{Deep Neural Networks}
In~\pref{sec:Overview}, we represented neural networks by \emph{graphs.} While
this is useful for intuition, in~\pref{sec:ANNs} we will talk about, e.g.,
\emph{octagons of layer weight matrices}, for which the graph representation
makes significantly less intuitive sense. Hence, for the rest of the paper we
will use an entirely equivalent \emph{matrix representation} for DNNs, which
will simplify the definitions, intuition, and proofs considerably.  With this
notation, we think of nodes as \emph{dimensions} and layers of nodes as
\emph{intermediate spaces.} We then define a \emph{layer} to be a
transformation from one intermediate space to another.

\begin{definition}
    A \emph{DNN layer from $n$ to $m$ dimensions} is a tuple $(\concr{W},
    \csigma)$ where $\concr{W}$ is an $m\times n$ matrix and $\csigma :
    \mathbb{R}\to \mathbb{R}$ is an arbitrarily-chosen \emph{activation
    function}.
\end{definition}

We will often abuse notation such that, for a vector $\vec{v}$,
$\csigma(\vec{v})$ is the vector formed by applying $\csigma$ to each
component of $\vec{v}$.

\begin{definition}
    A \emph{Deep Neural Network (DNN) with layer sizes $s_0, s_1, \ldots, s_n$}
    is a collection of $n$ DNN layers $(\cWi{1}, \csigmai{1}), \ldots,
    (\cWi{n}, \csigmai{n})$, where the $(\cWi{i}, \csigmai{i})$ layer is from
    $s_{i-1}$ to $s_i$ dimensions.
\end{definition}

Every DNN has a corresponding \emph{function,} defined below.
\begin{definition}
    \label{def:DNN}
    Given a DNN from $s_0$ to $s_n$ dimensions with layers $(\cWi{i},
    \csigmai{i})$, the \emph{function corresponding to the DNN} is the
    function $f : \mathbb{R}^{s_0} \to \mathbb{R}^{s_n}$ given by $f(\vec{v}) =
    \vec{v}^{(n)}$, where $\vec{v}^{(i)}$ is defined inductively by
    $\vec{v}^{(0)} = \vec{v}$ and $\vec{v}^{(i)} =
    \csigmai{i}(\cWi{i}(\vec{v}^{(i-1)}))$.
\end{definition}
Where convenient, we will often refer to the corresponding function as the DNN
or vice-versa.

\begin{example}
    The DNN $N_1$ from~\pref{fig:OverviewDNN}, when using the $\sigma$ hidden-layer
    activation function, is represented by the layers
    $\left({\scriptsize\begin{bmatrix} 1 \\ -1 \end{bmatrix}}, \sigma\right)$
    and
    $\left({\scriptsize\begin{bmatrix} 1 & 1 \\ 1 & 0 \\ 0 & 1 \end{bmatrix}}, I\right)$.
    The \emph{function corresponding to the DNN} is given by $N_1(x_1) =
    {\scriptsize\begin{bmatrix}
        1 & 1 \\
        1 & 0 \\
        0 & 1
    \end{bmatrix}}
    \sigma\left(
    {\scriptsize
        \begin{bmatrix}
            1 \\
            -1
        \end{bmatrix}
        \begin{bmatrix}
            x_1
        \end{bmatrix}}
    \right)$.
\end{example}

\subsection{Common Activation Functions}
There are a number of commonly-used activation functions, listed below.
\begin{definition}
    \label{def:Activations}
    The \emph{Leaky Rectified Linear Unit} (LReLU)~\cite{maasrectifier},
    \emph{Rectified Linear Unit} (ReLU), \emph{Hyperbolic Tangent} ($\tanh$),
    and \emph{Threshold} ($\mathrm{thresh}$) activation functions are defined:
    \[
        \begin{aligned}
            \mathrm{LReLU}(x; \mathtt{c}) &\coloneqq
            \begin{cases}
                x &x \geq 0 \\
                \mathtt{c}x &x < 0
            \end{cases},
            &&\qquad\mathrm{ReLU}(x) \coloneqq \mathrm{LReLU}(x; 0),
            \\
            \tanh &\coloneqq \frac{e^{2x}-1}{e^{2x} + 1},
            &&\mathrm{thresh}(x; \mathtt{t}, \mathtt{v}) \coloneqq
            \begin{cases}
                x &\text{if } x \geq \mathtt{t} \\
                \mathtt{v} &\text{otherwise}
            \end{cases}.
        \end{aligned}
    \]
    Here $\mathrm{LReLU}$ and $\mathrm{thresh}$ actually represent
    \emph{families} of activation functions parameterized by the constants $\mathtt{c},
    \mathtt{t}, \mathtt{v}$. The constants used varies between networks. $\mathtt{c} = 0$ is a common
    choice for the $\mathrm{LReLU}$ parameter, hence the explicit definition of
    $\mathrm{ReLU}$.
\end{definition}

All of these activation functions are present in standard deep-learning
toolkits, such as Pytorch~\cite{Pytorch}. Libraries such as Pytorch also enable
users to implement new activation functions. This variety of activation
functions used in practice will motivate our study of necessary and sufficient
conditions on the activation function to permit sound over-approximation.

\section{Abstract Neural Networks}
\label{sec:ANNs}

\newcommand\pr[2]{\mathcal{P}(\mathbb{R}^{#1\times #2})}

In this section, we formalize the syntax and semantics of Abstract Neural Networks (ANNs).
We also present two types of ANNs: Interval Neural Networks (INNs) and Octagon Neural Networks (ONNs). 

An ANN is like a DNN except the weights in each layer are represented by 
an abstract value in some abstract domain. This is formalized below.

\begin{definition}
    An \emph{$n\times m$ weight set abstract domain} is a lattice $\AbsDomain$
    with Galois connection $(\aalpha_{\AbsDomain}, \cgamma_{\AbsDomain})$ with
    the powerset lattice $\pr{n}{m}$ of $n\times m$ matrices.
\end{definition}

\begin{definition}
    An \emph{ANN layer from $n$ to $m$ dimensions} is a triple $(\AbsDomain,
    \AbsElem, \csigma)$ where $\AbsElem$ is a member of the weight set
    abstraction $\AbsDomain$ and $\csigma : \mathbb{R}\to \mathbb{R}$ is an
    arbitrarily-chosen \emph{activation function}.
\end{definition}

\noindent Thus, we see that each ANN layer $(\AbsDomain, \AbsElem, \csigma)$ is associated
with a set of weights $\cgamma_{\AbsDomain}(\AbsElem)$. Finally, we can
define the notion of an ANN:
\begin{definition}
    An \emph{Abstract Neural Network (ANN) with layer sizes $s_0, s_1, \ldots,
    s_n$} is a collection of $n$ ANN layers $(\AbsDomaini{i}, \AbsElemi{i},
    \csigmai{i})$, where the $i$th layer is from $s_{i-1}$ to $s_i$ dimensions.
\end{definition}

We consider the output of the ANN to be the set of outputs of all
\emph{instantiations} of the ANN into a DNN, as illustrated
in~\pref{fig:ANNSemantics}.
\begin{definition}
    We say a DNN with layers $(\cWi{i}, \csigmai{i})$ is \emph{an
    instantiation of} an ANN $\abstr{T}$ with layers $(\AbsDomaini{i},
    \AbsElemi{i}, \csigmai{i})$ if each $\cWi{i} \in
    \cgamma_{\AbsDomaini{i}}(\AbsElemi{i})$.  The set of all DNNs that are
    instantiations of an ANN $\abstr{T}$ is given by $\cgamma(\abstr{T})$.
\end{definition}

The semantics of an ANN naturally lift those of the DNN instantiations.
\begin{definition}
    \label{def:ANNSemantics}
    For an ANN $\abstr{T}$ from $s_0$ to $s_n$ dimensions, the \emph{function
    corresponding to $\abstr{T}$} is the set-valued function $\abstr{T} :
    \mathbb{R}^{s_0} \to \mathcal{P}(\mathbb{R}^{s_n})$ defined by
    \\
    $\abstr{T}(\vec{v}) \coloneqq \{ \concr{g}(\vec{v}) \mid \concr{g} \in
    \cgamma(\abstr{T}) \}$.
\end{definition}

\newcommand\tikzwv[2]{
    \node (w#1) at (#2, 0) {$\vec{w}^{(#1)}$};
    \node (v#1) at (#2+1, 0) {$\vec{v}^{(#1)}$};
}
\newcommand\instrow[2]{
    \node (v_#2_0) at (0, #1) {$\vec{v}$};
    \foreach \i in {1,2,...,3}{
        \node (w_#2_\i) at (4*\i-2, #1) {$\vec{w}^{(#2, \i)}$};
        \node (v_#2_\i) at (4*\i,   #1) {$\vec{v}^{(#2, \i)}$};
        \draw[->] (v_#2_\number\numexpr \i-1 \relax) -- (w_#2_\i) node[midway,below] (H_#2_\i) {$\concr{H^{(#2, \i)}}$};
        \draw[->] (w_#2_\i) -- (v_#2_\i) node[midway,below] (s_#2_\i) {$\csigmai{\i}$};

    }
    \node[yshift=-1mm] (v_#2_inT) at (13, #1) {$\in \abstr{T}(\vec{v})$};
}
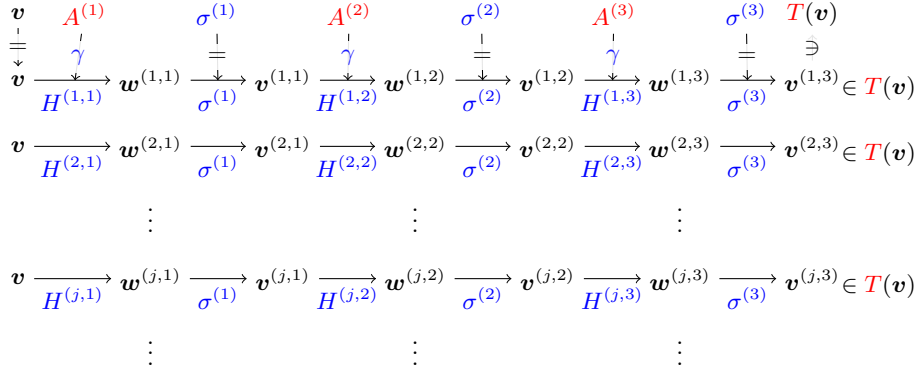
\begin{figure}[t]
    \centering
\begin{tikzpicture}[scale=0.88]
    \node (v0) at (0, 0) {$\vec{v}$};

    \foreach \i in {1,2,...,3}{
        \node (A\i) at (4*\i-3, 0) {$\AbsElemi{\i}$};
        \node (s\i) at (4*\i-1, 0) {$\csigmai{\i}$};
    }

    \instrow{-1}{1}
    \instrow{-2}{2}
    \instrow{-4}{j}
    \node (vdots) at (6, -3) {$\vdots$};
    \node (vdots) at (2, -3) {$\vdots$};
    \node (vdots) at (10, -3) {$\vdots$};
    \node (vdots) at (6, -5) {$\vdots$};
    \node (vdots) at (2,  -5) {$\vdots$};
    \node (vdots) at (10, -5) {$\vdots$};

    \draw[->] (v0) -- (v_1_0)
    node[midway,fill=white,text opacity=1,opacity=0.9] {$=$};
    \foreach \i in {1,2,...,3}{
        \draw[->] (A\i) -- (H_1_\i)
        node[midway,fill=white,text opacity=1,opacity=0.9] {$\cgamma$};
        \draw[->] (s\i) -- (s_1_\i)
        node[midway,fill=white,text opacity=1,opacity=0.9] {$=$};
    }

    \node (out) at (12, 0) {$\abstr{T}(\vec{v})$};
    \draw[->] (v_1_3) -- (out)
    node[midway,fill=white,text opacity=1,opacity=0.9] {$\ni$};
\end{tikzpicture}
    \caption{
        Visualization of ANN semantics for a 3-layer ANN $\abstr{T}$ (first
        row). Different DNN \emph{instantiations} (other rows)
        of $\abstr{T}$ are formed by replacing each abstract weight matrix $\AbsElemi{i}$
        by some concrete weight matrix $\concr{H^{(j, i)}} \in
        \cgamma(\AbsElemi{i})$. $\vec{v}^{(j, 3)}$ is the output of each
        instantiation on the input $\vec{v}$.
        The set of all such outputs producable by some valid
        instantiation is taken to be the output $\abstr{T}(\vec{v})$ of the ANN
        on vector $\vec{v}$.
    }
    \label{fig:ANNSemantics}
\end{figure}

Space constraints prevent us from defining a full Galois
connection here, however one can be established between the lattice of ANNs of
a certain \emph{architecture} and the powerset of DNNs of the same
architecture.

The definition of an ANN above is agnostic to the actual abstract domain(s)
used. For expository purposes, we now define two particular types of ANNs:
\emph{Interval Neural Networks} (INNs) and \emph{Octagon Neural Networks}
(ONNs).

\begin{definition}
    \label{def:INNs}
    An \emph{Interval Neural Network (INN)} is an ANN with layers \\
    $(\AbsDomaini{i}, \AbsElemi{i}, \csigmai{i})$, where each $\AbsDomaini{i}$
    is an \emph{interval hull domain}~\cite{DBLP:conf/popl/CousotC77}. The
    interval hull domain represents sets of matrices by their
    \emph{component-wise interval hull.}
\end{definition}
Notably, the definition of INN in \prabhakarcite{} is equivalent to the above, except
that they further assume every activation function $\csigmai{i}$ is the
$\mathrm{ReLU}$ function.

\begin{example}
    We first demonstrate the interval hull domain:
    $\cgamma_{\IntDom}\left(
        {\scriptsize\begin{bmatrix}
            [-1, 1] & [0, 2] \\
            [-3, -2] & [1, 2] \\
        \end{bmatrix}}
    \right)
    =
    \left\{
        {\scriptsize\begin{bmatrix}
            a & b \\
            c & d \\
        \end{bmatrix}
        \mid
        a \in [-1, 1], b \in [0, 2],
        c \in [-3, -2], d \in [1, 2]}
    \right\}$.
    We can thus define a two-layer INN
    $f(\vec{v})
    \coloneqq
    \begin{bmatrix}
        [0, 1] & [0, 1] \\
    \end{bmatrix}
    \mathrm{ReLU}\left(
    {\scriptsize
    \begin{bmatrix}
        [-1, 1] & [0, 2] \\
        [-3, -2] & [1, 2] \\
    \end{bmatrix}}
    \vec{v}\right)$.
    We can instantiate this network in a variety of ways, for example
    $g(\vec{v}) \coloneqq
    \begin{bmatrix}
        0.5 & 1 \\
    \end{bmatrix}
    \mathrm{ReLU}\left(
    {\scriptsize\begin{bmatrix}
        0 & 2 \\
        -2.5 & 1.5 \\
    \end{bmatrix}}
    \vec{v}\right) \in \cgamma(f)$.
    Taking arbitrarily $(1, 1)^T$ as an example input, we have
    $g((1, 1)^T) = \begin{bmatrix} 1 \end{bmatrix} \in f((1, 1)^T)$.
    In fact, $f((1, 1)^T)$ is the set of \emph{all} values that can be achieved
    by such instantiations, which in this case is the set given by
    $f((1, 1)^T) = \begin{bmatrix} [0, 3] \end{bmatrix}$.
\end{example}

\begin{definition}
    \label{def:ONNs}
    An \emph{Octagon Neural Network (ONN)} is an ANN with layers
    $(\AbsDomaini{i}, \AbsElemi{i}, \csigmai{i})$, where each
    $\AbsDomaini{i}$ is an \emph{octagon hull domain}~\cite{mine2006octagon}.
    The octagon hull domain represents sets of matrices by octagons in the
    space of their components.
\end{definition}

\begin{example}
    Octagons representing a set of $n\times m$ matrices can be thought of
    exactly like an octagon in the vector space $\mathbb{R}^{n\cdot m}$.
    Unfortunately, this is particularly difficult to visualize in higher
    dimensions, hence in this example we will stick to the case where $nm = 2$.

    Let $O_1, O_2$ be octagons such that
    \[
        \begin{aligned}
            \cgamma_{\OctDom}(O_1) &=
            \left\{
                \begin{bmatrix}
                    a \\
                    b
                \end{bmatrix}
                \mid
                a - b \leq 1, -a + b \leq 1, a + b \leq 2, -a - b \leq 2
            \right\},
            \\
            \cgamma_{\OctDom}(O_2) &=
            \left\{
                \begin{bmatrix}
                    a & b \\
                \end{bmatrix}
                \mid
                a - b \leq 2, -a + b \leq 3, a + b \leq 4, -a - b \leq 5
            \right\}.
        \end{aligned}
    \]
    We can thus define a two-layer ONN $f(\vec{v}) \coloneqq O_2
    \mathrm{ReLU}\left( O_1 \vec{v}\right)$.  One instantiation 
    of this ONN $f$ is the DNN
    $g(\vec{v}) \coloneqq
    \begin{bmatrix}
        3 & 1 \\
    \end{bmatrix}
    \mathrm{ReLU}\left(
    {\scriptsize
    \begin{bmatrix}
        0.5 \\
        1.5 \\
    \end{bmatrix}}
    \vec{v}\right) \in \cgamma(f)$.
    We can confirm that $g(1) = \begin{bmatrix} 3 \end{bmatrix} \in f(1)$.
\end{example}

We can similarly define Polyhedra Neural Networks (PNNs) using the polyhedra
domain~\cite{DBLP:conf/popl/CousotH78}.

\section{Layer-Wise Abstraction Algorithm}
\label{sec:Algorithm}
Given a large DNN, how might we construct a smaller ANN which soundly
\emph{over-approximates} that DNN? We define over-approximation formally below.
\begin{definition}
    An ANN $\abstr{T}$ \emph{over-approximates} a DNN $\concr{N}$ if, for every
    $\vec{v} \in \mathbb{R}^n$, $\concr{N}(\vec{v}) \in \abstr{T}(\vec{v})$.
\end{definition}

\begin{remark}
    By~\pref{def:ANNSemantics}, then, $\abstr{T}$ over-approximates $\concr{N}$
    if, for every $\vec{v}$ we can find some instantiation $\concr{T_{\vec{v}}}
    \in \cgamma(\abstr{T})$ such that $\concr{T_{\vec{v}}}(\vec{v}) =
    \concr{N}(\vec{v})$.
\end{remark}

\pref{alg:quotient} constructs a small ANN that, under certain assumptions
discussed in~\pref{sec:Overview}, soundly over-approximates the large DNN
given. The basic idea is to \emph{merge} groups of dimensions together, forming
an ANN where each dimension in the ANN represents a collection of dimensions in
the original DNN.
We formalize the notion of ``groups of dimensions'' as a \emph{layer-wise
partitioning.}
\begin{definition}
    Given a DNN with layer sizes $s_0, s_1, \ldots, s_n$, a \emph{layer-wise
    partitioning $\nping$ of the network} is a set of partitionings $\lping{0},
    \lping{1}, \ldots, \lping{n}$ where each $\lping{i}$ partitions $\{ 1, 2,
    \ldots, s_i \}$. For ease of notation, we will write partitionings with set
    notation but assume they have some intrinsic ordering for indexing.
\end{definition}

\begin{remark}
    To maintain the same number of input and output dimensions in our ANN and
    DNN, we assume $\lping{0} = \{\{1\}, \{2\}, \ldots, \{s_0\}\}$ and
    $\lping{n} = \{\{1\}, \{2\}, \ldots, \{s_n\}\}$.
\end{remark}

\begin{example}
    \label{exa:Partitioning}
    Consider the DNN corresponding to the function

    $f(x_1) =
    {\scriptsize\begin{bmatrix}
        1 & 1 \\
        1 & 0 \\
        0 & 1
    \end{bmatrix}}
    \mathrm{ReLU}\left(
        {\scriptsize\begin{bmatrix}
            1 \\
            -1
        \end{bmatrix}
        \begin{bmatrix}
            x_1
        \end{bmatrix}}
    \right)$.
    The layer sizes are $s_0 = 1, s_1 = 2, s_2 = 3$. Hence, one valid layer-wise
    partitioning is to merge the two inner dimensions:
    $\lping{0} = \{ \{1\} \}
    \quad
    \lping{1} = \{ \{1, 2\} \}
    \quad
    \lping{2} = \{ \{1\}, \{2\}, \{3\} \}$.
    Here we have, e.g., $\lping{0}_1 = \{1\}$, $\lping{1}_1 = \{1, 2\}$, and
    $\lping{2}_3 = \{3\}$.
\end{example}

\SetKw{returnKw}{return}
\begin{figure}[t]
    \small
    \begin{minipage}{0.48\textwidth}
    \begin{algorithm}[H]
        \DontPrintSemicolon
        \KwIn{
            Matrix $M$. Partitionings $\ping^{in}$, $\ping^{out}$ with
            $\abs{\ping^{in}} = k$. Abstract domain $\AbsDomain$.
        }
        \KwOut{
            Abstract element representing all merges of $M$.
        }
        $S \gets \{ \}$\;
        $w \gets (\abs{\ping^{in}_1}, \abs{\ping^{in}_{2}}, \ldots, \abs{\ping^{in}_{k}})$\;
        \For{$C \in \PCMs(\ping^{in})$}{
            \For{$D \in \PCMs(\ping^{out})$}{
                $S \gets S \cup \{\ScaleCols(D^T M C, w))\}$\;
            }
        }
        \returnKw{$\aalpha_{\AbsDomain}(S)$}\;

        \caption{$\haalpha(M, \ping^{in}, \ping^{out}, \AbsDomain)$}
        \label{alg:ahat}
    \end{algorithm}
    \end{minipage}
    \begin{minipage}{0.48\textwidth}
    \begin{algorithm}[H]
        \DontPrintSemicolon
        \KwIn{
            Matrix $M$. Partitionings $\ping^{in}$, $\ping^{out}$ with
            $\abs{\ping^{in}} = k$. Abstract domain $\AbsDomain$.
        }
        \KwOut{
            Abstract element representing all binary merges of $M$
        }
        $S \gets \{ \}$\;
        $w \gets (\abs{\ping^{in}_1}, \abs{\ping^{in}_{2}}, \ldots, \abs{\ping^{in}_{k}})$\;
        \For{$C \in \BinPCMs(\ping^{in})$}{
            \For{$D \in \BinPCMs(\ping^{out})$}{
                $S \gets S \cup \{\ScaleCols(D^T M C, w))\}$\;
            }
        }
        \returnKw{$\aalpha_{\AbsDomain}(S)$}\;

        \caption{$\haalpha_{bin}(M, \ping^{in}, \ping^{out}, \AbsDomain)$}
        \label{alg:ahatbin}
    \end{algorithm}
    \end{minipage}
\vspace{-2ex}
\end{figure}
\begin{figure}[t]
    \small
\begin{algorithm}[H]
    \DontPrintSemicolon
    \KwIn{
        DNN $\concr{N}$ consisting of $n$ layers $(\cWi{i}, \csigmai{i})$
        with each $\csigmai{i} \in \cSigma$.
        Layer-wise partitioning $\nping$ of $\concr{N}$.
        List of $n$ abstract weight domains $\AbsDomaini{i} \in \AbsCollection$.
    }
    \KwOut{
        An ANN with layers $(\AbsDomaini{i}, \AbsElemi{i}, \csigmai{i})$ where
        $\AbsElemi{i} \in \AbsDomaini{i} \in \AbsCollection$.
    }
    $\abstr{A} \gets [\quad]$\;
    \For{$i \in \{1, 2, \ldots, n\}$}{
        $\AbsElemi{i} \gets \haalpha(\cWi{i}, \lping{i-1}, \lping{i}, \AbsDomaini{i})$\;
        $\abstr{A}.\mathrm{append}\big((\AbsDomaini{i}, \AbsElemi{i}, \csigmai{i})\big)$\;
    }
    \returnKw{$\abstr{A}$}\;

    \caption{$\ALW$}
    \label{alg:quotient}
\end{algorithm}
\vspace{-2ex}
\end{figure}

Our layer-wise abstraction algorithm is shown in~\pref{alg:quotient}. For each
layer in the DNN, we will call~\pref{alg:ahat} to abstract the set of
\emph{mergings} of the layer's weight matrix. This abstract element becomes the
abstract weight $\AbsElemi{i}$ for the corresponding layer in the constructed
ANN.

The functions $\PCMs$ and $\ScaleCols$ are defined more precisely below.

\begin{definition}
    Let $P$ be some partition, i.e., non-empty subset, of $\{ 1, 2, \ldots, n
    \}$. Then a vector $\vec{c} \in \mathbb{R}^n$ is a \emph{partition
    combination vector} (PCV) if (i) each component $c_i$ is non-negative, (ii)
    the components of $c_i$ sum to one, and~(iii) $c_i = 0$ whenever $i \not\in
    P$.
\end{definition}

\newcommand*{\vertbar}{\rule[-1ex]{0.5pt}{2.5ex}}
\begin{definition}
    Given a partitioning $\ping$ of $\{ 1, 2, \ldots, n \}$ with $\abs{\ping} =
    k$, a \emph{partitioning combination matrix} (PCM) is a matrix
    $
        {\scriptsize C =
        \begin{bmatrix}
            \vertbar & \vertbar & & \vertbar \\
            \vec{c_1} & \vec{c_2} & \cdots & \vec{c_k} \\
            \vertbar & \vertbar & & \vertbar
        \end{bmatrix}},
    $
    where each $\vec{c_i}$ is a PCV of partition $\ping_i$. We refer to the set
    of all such PCMs for a partitioning $\ping$ by $\PCMs(\ping)$.
\end{definition}

\begin{definition}
    A PCM is \emph{binary} if each entry is either 0~or~1. We refer to the set
    of all binary PCMs for a partitioning $\ping$  as
    $\BinPCMs(\ping)$.
\end{definition}

\begin{definition}
    For an $n\times m$ matrix $M$, PCM $C$ of partitioning $\ping^{in}$ of
    $\{ 1, 2, \ldots, m \}$, and PCM $D$ for partitioning $\ping^{out}$ of $\{
        1, 2, \ldots, n \}$, we call $D^T M C$ a \emph{merging of~$M$.}
\end{definition}

The $j$th column in $MC$ is a convex combination of the columns of $M$ that
belong to partition $\ping^{in}_j$, weighted by the $j$th column of $C$.
Similarly, the $i$th row in $D^T M$ is a convex combination of the rows in $M$
that belong to partition $\ping^{out}_i$. In total, the $i, j$th entry of
merged matrix $D^T M C$ is a convex combination of the entries of $M$ with
indices in $\ping^{out}_i \times \ping^{in}_j$. This observation will lead
to~\pref{thm:BinPCMs} in~\pref{sec:AlgorithmComputable}.

\begin{definition}
    Given a matrix $M$, the \emph{column-scaled matrix} formed by weights $w_1,
    w_2, \ldots, w_k$ is the matrix with entries given component-wise by
    \\
    $\ScaleCols(M, (w_1, \ldots, w_k))_{i,j} \coloneqq M_{i,j}w_{j}$.
\end{definition}

Intuitively, column-scaling is needed because what were originally $n$
dimensions contributing to an input have been collapsed into a single
representative dimension. This is demonstrated nicely for the specific case of
Interval Neural Network and ReLU activations by Figures~3~and~4 in~\prabhakarcite{}.

\begin{example}
    Given the matrix
    $M =
        {\scriptsize\begin{bmatrix}
            m_{1, 1} & m_{1, 2} & m_{1, 3} \\
            m_{2, 1} & m_{2, 2} & m_{2, 3} \\
            m_{3, 1} & m_{3, 2} & m_{3, 3} \\
            m_{4, 1} & m_{4, 2} & m_{4, 3} \\
        \end{bmatrix}}$,
    partitioning $\lping{0} = \{ \{1,3\}, \{2\}\}$ of the input dimensions and
    $\lping{1} = \{\{2,4\},\{1,3\}\}$ of the output dimensions, we can define a
    PCM for $\lping{0}$ as
    ${\scriptsize
        C \coloneqq \begin{bmatrix}
            0.25 & 0 \\
            0 & 1 \\
            0.75 & 0 \\
        \end{bmatrix}}$
    and a PCM for $\lping{1}$ as:
    ${\scriptsize
        D \coloneqq \begin{bmatrix}
            0 & 0.99 \\
            0.4 & 0 \\
            0 & 0.01 \\
            0.6 & 0
        \end{bmatrix}}$.
    We can then compute the \emph{column--merged matrix}
    ${\scriptsize
        M C
        =
        \begin{bmatrix}
            0.25m_{1, 1} + 0.75m_{1, 3} & m_{1, 2} \\
            0.25m_{2, 1} + 0.75m_{2, 3} & m_{2, 2} \\
            0.25m_{3, 1} + 0.75m_{3, 3} & m_{3, 2} \\
            0.25m_{4, 1} + 0.75m_{4, 3} & m_{4, 2} \\
        \end{bmatrix}}$,
    and furthermore the \emph{column-row--merged matrix}
    \\
    ${\scriptsize
        D^T M C
        =
        \begin{bmatrix}
            0.4(0.25m_{2, 1} + 0.75m_{2, 3})
            + 0.6(0.25m_{4, 1} + 0.75m_{4, 3}) &
            0.4m_{2, 2} + 0.6m_{4, 2}
            \\
            0.99(0.25m_{1, 1} + 0.75m_{1, 3})
            + 0.01(0.25m_{3, 1} + 0.75m_{3, 3}) &
            0.99m_{1, 2} + 0.01m_{3, 2} \\
        \end{bmatrix}}$.

    Finally, we can column-scale this matrix like so:
    \[
        \begin{aligned}
        \scriptsize
            &\ScaleCols(D^T M C, (2, 2)) \\
        =
            &\begin{bmatrix}
            0.8(0.25m_{2, 1} + 0.75m_{2, 3})
            + 1.2(0.25m_{4, 1} + 0.75m_{4, 3}) &
            0.8m_{2, 2} + 1.2m_{4, 2}
            \\
            1.98(0.25m_{1, 1} + 0.75m_{1, 3})
            + 0.02(0.25m_{3, 1} + 0.75m_{3, 3}) &
            1.98m_{1, 2} + 0.02m_{3, 2} \\
        \end{bmatrix}.
        \end{aligned}
    \]
\end{example}

\subsection{Computability}
\label{sec:AlgorithmComputable}
In general, there are an infinite number of mergings. Hence, to actually
compute $\haalpha$ (\pref{alg:ahat}) we need some non-trivial way to compute the
abstraction of the infinite set of mergings. If the abstract domain $\AbsDomaini{i}$ is \emph{convex}, it can
be shown that one only needs to iterate over the binary PCMs, of which there
are finitely many, producing a computationally feasible algorithm.

\begin{definition}
    \label{def:BetaConvex}
    A weight set abstract domain $\AbsDomain$ is \emph{convex} if, for any set
    $S$ of concrete values, $\cgamma_{\AbsDomain}(\aalpha_{\AbsDomain}(S))$ is
    convex.
\end{definition}
Many commonly-used abstractions --- including
intervals~\cite{DBLP:conf/popl/CousotC77}, octagons~\cite{mine2006octagon}, and
polyhedra~\cite{DBLP:conf/popl/CousotH78} --- are convex.

\begin{restatable}{theorem}{ThmBinPCMs}
    \label{thm:BinPCMs}
    If $\AbsDomain$ is convex, then
    $\haalpha(M, \ping^{in}, \ping^{out}, \AbsDomain) =
    \haalpha_{bin}(M, \ping^{in}, \ping^{out}, \AbsDomain)$.
\end{restatable}

\appendixproof{app:BinPCMs}

\begin{remark}
    Consider PCMs $C$ and $D$ corresponding to merged matrix $D^T \cWi{i} C$.
    We may think of $C$ and $D$ as vectors in the vector space of matrices.
    Then their outer product $D \otimes C$ forms a convex coefficient matrix of
    the binary mergings $R$ of $\cWi{i}$, such that $(D\otimes C)R =
    D^T\cWi{i}C$.
    From this intuition, it follows that the converse to~\pref{thm:BinPCMs}
    \emph{does not} hold, as every matrix $E$ cannot be decomposed into vectors
    $D\otimes C$ as described (i.e., not every matrix has rank 1).  Hence, the
    convexity condition may be slightly weakened.  However, we are not
    presently aware of any abstract domains that satisfy such a condition but
    not convexity.
\end{remark}

\begin{example}
    Let
    ${\scriptsize
        \cWi{i} =
        \begin{bmatrix}
            1 & -2 & 3 \\
            4 & -5 & 6 \\
            7 & -8 & 9 \\
        \end{bmatrix}}
    $
    and consider $\lping{i-1} = \{\{1, 2\}, \{3\}\}$ and $\lping{i} = \{\{1,
    3\}, \{2\}\}$. Then we have the binary PCMs
    ${\scriptsize
        \BinPCMs(\lping{i-1})
        = \left\{
            \begin{bmatrix} 1 & 0 \\ 0 & 0 \\ 0 & 1 \\ \end{bmatrix},
            \begin{bmatrix} 0 & 0 \\ 1 & 0 \\ 0 & 1 \\ \end{bmatrix}
          \right\}
    }$
    \\
    and
    ${\scriptsize
        \BinPCMs(\lping{i})
        = \left\{
            \begin{bmatrix} 1 & 0 \\ 0 & 1 \\ 0 & 0 \\ \end{bmatrix},
            \begin{bmatrix} 0 & 0 \\ 0 & 1 \\ 1 & 0 \\ \end{bmatrix}
          \right\}
    }.$ These correspond to the column-scaled binary mergings
    ${\scriptsize
        \left\{
            \begin{bmatrix}
                2 & 3 \\
                8 & 6 \\
            \end{bmatrix},
            \begin{bmatrix}
                -4 & 3 \\
                -10 & 6 \\
            \end{bmatrix},
            \begin{bmatrix}
                14 & 9 \\
                8 & 6 \\
            \end{bmatrix},
            \begin{bmatrix}
                -16 & 9 \\
                -10 & 6 \\
            \end{bmatrix}
        \right\}}
    $.

    We can take any PCMs such as
    ${\scriptsize C = \begin{bmatrix} 0.75 & 0 \\ 0.25 & 0 \\ 0 & 1 \\ \end{bmatrix}}$
    for $\lping{i-1}$ as well as
    ${\scriptsize D = \begin{bmatrix} 0.5 & 0 \\ 0 & 1 \\ 0.5 & 0 \\ \end{bmatrix}}$
    for $\lping{i}$, resulting in the scaled merging
    ${\scriptsize
    \ScaleCols(D^T \cWi{i} C, (2, 1)) = \begin{bmatrix} 3.5 & 6 \\ 3.5 & 6 \end{bmatrix}}$.
    According to~\pref{thm:BinPCMs}, we can write this as a convex combination
    of the four column-scaled binary merged matrices. In particular, we find
    the combination
    \[
        \scriptsize
        \begin{aligned}
            \begin{bmatrix}
                3.5 & 6 \\
                3.5 & 6 \\
            \end{bmatrix}
            =
            &(1.5/2)(1)(0.5)(1)
            \begin{bmatrix}
                2 & 3 \\
                8 & 6 \\
            \end{bmatrix}
            +
            (0.5/2)(1)(0.5)(1)
            \begin{bmatrix}
                -4 & 3 \\
                -10 & 6 \\
            \end{bmatrix} \\
            &+
            (1.5/2)(1)(0.5)(1)
            \begin{bmatrix}
                14 & 9 \\
                8 & 6 \\
            \end{bmatrix}
            +
            (0.5/2)(1)(0.5)(1)
            \begin{bmatrix}
                -16 & 9 \\
                -10 & 6 \\
            \end{bmatrix}.
        \end{aligned}
    \]

    We can confirm that this is a convex combination, as
    \[
        \scriptsize
        (1.5/2)(1)(0.5)(1)
        +
        (0.5/2)(1)(0.5)(1)
        +
        (1.5/2)(1)(0.5)(1)
        +
        (0.5/2)(1)(0.5)(1)
        = 1.
    \]

    Because we can find such a convex combination for any such non-binary
    merging in terms of the binary ones, and because the abstract domain is
    assumed to be convex, including only the binary mergings will ensure that
    \emph{all} mergings are represented by the abstract element $\AbsElemi{i}$.
\end{example}

\subsection{Walkthrough Example}
\begin{example}
    \label{exa:AlgWalkthrough}
    Consider again the DNN from~\pref{exa:Partitioning} corresponding to\\
    $f(x_1) =
    {\scriptsize
    \begin{bmatrix}
        1 & 1 \\
        1 & 0 \\
        0 & 1
    \end{bmatrix}}
    \csigma\left(
        {\scriptsize\begin{bmatrix}
            1 \\
            -1
        \end{bmatrix}
        \begin{bmatrix}
            x_1
        \end{bmatrix}}
    \right)$,
    the partitioning
    $\lping{0} = \{\{1\}\}$,
    $\lping{1} = \{ \{1, 2\}\}$,
    $\lping{2} = \{\{1\}, \{2\}, \{3\}\}$,
    which collapses the two hidden dimensions, and assume the abstract domains
    $\AbsDomaini{i}$ are all convex.

    For the input layer, we have $w = (1)$, because the only partition in
    $\lping{0}$ has size $1$.  Similarly, the only binary PCM for $\lping{0}$
    is $C = \begin{bmatrix} 1 \end{bmatrix}$.  However, there are two binary
    PCMs for $\lping{1}$, namely
    $D = {\scriptsize\begin{bmatrix} 1 \\ 0 \end{bmatrix}}$ or
    $D = {\scriptsize\begin{bmatrix} 0 \\ 1 \end{bmatrix}}$. These correspond to the binary
    merged matrices $\begin{bmatrix} 1 \end{bmatrix}$ and
    $\begin{bmatrix} -1 \end{bmatrix}$. Hence, we get
    $\AbsElemi{1} = \aalpha_{\AbsDomaini{1}}(
    \{\begin{bmatrix} 1 \end{bmatrix}, \begin{bmatrix} -1 \end{bmatrix}\})$,
    completing the first layer.

    For the output layer, we have $w = (2)$, because the only partition in
    $\lping{1}$ contains \emph{two} nodes.  Hence, the column scaling will need
    to play a role: because we have merged two dimensions in the domain, we
    should interpret any value from that dimension as being from \emph{both} of
    the dimensions that were merged. We have two binary mergings, namely
    ${\scriptsize\begin{bmatrix} 1 \\ 1 \\ 0 \end{bmatrix}}$
    and
    ${\scriptsize\begin{bmatrix} 1 \\ 0 \\ 1 \end{bmatrix}}$,
    which after rescaling gives us
    $\AbsElemi{2} = \aalpha_{\AbsDomaini{2}}\left(
    \left\{
    {\scriptsize\begin{bmatrix} 2 \\ 2 \\ 0 \end{bmatrix},
     \begin{bmatrix} 2 \\ 0 \\ 2 \end{bmatrix}
    }\right\}
    \right)$.

    In total then, the returned ANN can be written $(\AbsElemi{1}, \csigma)$,
    $(\AbsElemi{2}, \concr{x \mapsto x})$ or in a more functional notation as
    $g(x) = \AbsElemi{2}\csigma(\AbsElemi{1}x)$, where in either case
    $\AbsElemi{1} = \aalpha_{\AbsDomaini{1}}(
    \{\begin{bmatrix} 1 \end{bmatrix}, \begin{bmatrix} -1 \end{bmatrix}\})$,
    and
    $\AbsElemi{2} = \aalpha_{\AbsDomaini{2}}\left(
    \left\{
    {\scriptsize\begin{bmatrix} 2 \\ 2 \\ 0 \end{bmatrix},
     \begin{bmatrix} 2 \\ 0 \\ 2 \end{bmatrix}
    }\right\}
    \right)$.
\end{example}

Note in particular that, while the operation of the algorithm was agnostic to
the exact abstract domains $\AbsCollection$ and activation functions $\cSigma$
used, the semantics of the resulting ANN depend \emph{entirely} on these.
Hence, correctness of the algorithm will depend on the abstract domain and
activation functions satisfying certain conditions. We will discuss this
further in~\pref{sec:Examples}.

\section{Layer-Wise Abstraction: Instantiations and Examples}
\label{sec:Examples}

This section examines a number of examples. For some DNNs,~\pref{alg:quotient}
will produce a soundly over-approximating ANN. For others, the ANN will
provably \emph{not} over-approximate the given DNN.  We will generalize these
examples to necessary and sufficient conditions on the activation functions
$\concr{\Sigma}$ used in order for \\ $\ALW$ to soundly over-approximate
$\concr{N}$.

\subsection{Interval Hull Domain with ReLU Activation Functions}
\label{sec:ExampleINNReLU}
Consider again the DNN from~\pref{exa:AlgWalkthrough} given by
$f(x_1) =
{\scriptsize\begin{bmatrix}
    1 & 1 \\
    1 & 0 \\
    0 & 1
\end{bmatrix}}
\mathrm{ReLU}\left(
    {\scriptsize\begin{bmatrix}
        1 \\
        -1
    \end{bmatrix}
    \begin{bmatrix}
        x_1
    \end{bmatrix}}
\right)$ and partitioning which merges the two intermediate dimensions. Using
the interval hull domain in~\pref{exa:AlgWalkthrough} gives the corresponding
INN:
$g(x_1) =
{\scriptsize\begin{bmatrix}
    [2, 2] \\
    [0, 2] \\
    [0, 2]
\end{bmatrix}}
\mathrm{ReLU}\left(
    {\scriptsize\begin{bmatrix}
        [-1, 1]
    \end{bmatrix}
    \begin{bmatrix}
        x_1
    \end{bmatrix}}
\right)$.

In fact, because the ReLU activation function and interval domain was used, it
follows from the results of~\prabhakarcite{} that $g$ in fact over-approximates $f$.
To see this, consider two cases. If $x_1 > 0$, then the second component in the
hidden dimension of $f$ will always become $0$ under the activation function.
Hence,
$f(x_1) =
{\scriptsize\begin{bmatrix}
    1 \\
    1 \\
    0
\end{bmatrix}}
\mathrm{ReLU}\left(
    {\scriptsize\begin{bmatrix}
        1
    \end{bmatrix}
    \begin{bmatrix}
        x_1
    \end{bmatrix}}
\right)
=
{\scriptsize
\begin{bmatrix}
    2 \\
    2 \\
    0
\end{bmatrix}}
\mathrm{ReLU}\left(
    {\scriptsize\begin{bmatrix}
        0.5
    \end{bmatrix}
    \begin{bmatrix}
        x_1
    \end{bmatrix}}
\right)$,
which is a valid instantiation of the weights in $g$. Otherwise, if $x_1 \leq
0$, we find
$f(x_1) =
{\scriptsize\begin{bmatrix}
    2 \\
    0 \\
    2
\end{bmatrix}}
\mathrm{ReLU}\left(
    {\scriptsize\begin{bmatrix}
        -0.5
    \end{bmatrix}
    \begin{bmatrix}
        x_1
    \end{bmatrix}}
\right)$,
which is again a valid instantiation. Hence in all cases, the true output
$f(x_1)$ can be made by some valid instantiation of the weights in $g$.
Therefore, $f(x_1) \in g(x_1)$ for all $x_1$ and so $g$ over-approximates $f$.

\subsubsection{Sufficiency Condition}
\label{sec:ExampleSoundness}
The soundness of this particular instantiation can be generalized to a
sufficiency theorem,~\pref{thm:Sound}, for soundness of the layer-wise
abstraction algorithm. Its statement relies on the activation function
satisfying the \emph{weakened intermediate value property,} which is defined
below:
\begin{definition}
    \label{def:WIVP}
    A function $f : \mathbb{R} \to \mathbb{R}$ satisfies the \emph{Weakened
    Intermediate Value Property (WIVP)} if, for every $a_1 \leq a_2 \leq \cdots
    \leq a_n \in \mathbb{R}$, there exists some $b \in [a_1, a_n]$ such that
    $f(b) = \frac{\sum_i f(a_i)}{n}$.
\end{definition}
Every continuous function satisfies the IVP and hence the WIVP.  Almost all
commonly-used activation functions, except for $\mathrm{thresh}$, are
continuous and, therefore, satisfy the WIVP. However, the WIVP is not equivalent
to the IVP, as the below proof shows by constructing a function $f$ such that
$f((a, b)) = \mathbb{Q}$ for any non-empty open interval $(a, b)$.
\appendixproof{app:WIVP}

We now state the soundness theorem below, which is proved
in~\pref{sec:ProveSound}.
\begin{restatable}{theorem}{ThmSound}
    \label{thm:Sound}
    Let $\AbsCollection$ be a set of weight set abstract domains and
    $\concr{\Sigma}$ a set of activation functions. Suppose
    (i) each $\csigma \in \concr{\Sigma}$ has entirely non-negative outputs,
    and
    (ii) each $\csigma \in \concr{\Sigma}$ satisfies the Weakened Intermediate
    Value Property (\pref{def:WIVP}).
    Then $\abstr{T} = \ALW$ (\pref{alg:quotient}) soundly over-approximates the
    DNN $\concr{N}$.
\end{restatable}

\subsection{Interval Hull Domain with \emph{Leaky} ReLUs}
\label{sec:ExampleINNLeaky}
Something different happens if we slightly modify $f$ in~\pref{exa:AlgWalkthrough} to use an activation
function producing \emph{negative values} in the intermediate dimensions.  This
is quite common of activation functions like Leaky ReLU and $\tanh$, and was
not mentioned by~\prabhakarcite{}. For example, we will take the Leaky ReLU function
(\pref{def:Activations}) with $\mathtt{c} = 0.5$ and consider the DNN
$f(x_1) =
{\scriptsize\begin{bmatrix}
    1 & 1 \\
    1 & 0 \\
    0 & 1
\end{bmatrix}}
\mathrm{LReLU}\left(
    {\scriptsize\begin{bmatrix}
        1 \\
        -1
    \end{bmatrix}
    \begin{bmatrix}
        x_1
    \end{bmatrix}};
    0.5
\right)$.
Using the same partitioning gives us the INN
$g(x_1) =
{\scriptsize\begin{bmatrix}
    [2, 2] \\
    [0, 2] \\
    [0, 2]
\end{bmatrix}}
\mathrm{LReLU}\left(
    {\scriptsize\begin{bmatrix}
        [-1, 1]
    \end{bmatrix}
    \begin{bmatrix}
        x_1
    \end{bmatrix}};
    0.5
\right)$.

Surprisingly, this small change to the activation function in fact makes the
constructed ANN no longer over-approximate the original DNN. For example, note
that $f(1) = \begin{bmatrix} 0.5 & 1 & -0.5 \end{bmatrix}^T$
and consider $g(1)$. In~$g$, the output of the LReLU is one-dimensional, hence,
it will have either positive, negative, or zero sign. But no matter how the
weights in the final matrix are instantiated, every component of $g(1)$ will
have \emph{the same (or zero) sign}, and so $f(1) \not\in g(1)$, because $f(1)$
has mixed signs.

\subsubsection{Necessary Condition: Non-Negative Values}
\label{sec:NeedNonNegative}

We can generalize this counterexample to the following necessary
condition on soundness:
\begin{restatable}{theorem}{ThmNeedNonNegative}
    \label{thm:NeedNonNegative}
    Suppose some $\csigma \in \cSigma$ is an activation function with neither
    entirely non-negative nor entirely non-positive outputs, and every
    $\AbsDomain \in \AbsCollection$ is at least as precise as the interval hull
    abstraction.  Then there exists a neural network $\concr{N}$ that uses
    $\csigma$  and a partitioning $\nping$ such that $\abstr{T} = \ALW$ does
    not over-approximate $\concr{N}$.
\end{restatable}

\appendixproof{app:NeedNonNegative}

\subsubsection{Handling Negative Values}
\label{sec:ShiftNetwork}

Thankfully, there is a workaround to support sometimes-negative activation
functions. The constructive theorem below implies that a given DNN can be
modified into a \emph{shifted} version of itself such that the input-output
behavior on any arbitrary bounded region is retained, but the intermediate
activations are all non-negative.
\begin{restatable}{theorem}{ThmShiftNetwork}
    \label{thm:ShiftNetwork}
    Let $\concr{N}$ be a DNN and suppose that, on some input region $R$, the
    output of the activation functions are lower-bounded by a constant $C$.
    Then, there exists another DNN $\concr{N'}$, with at most one extra
    dimension per layer, which satisfies (i)~$\concr{N'}(x) = \concr{N}(x)$ for
    any $x \in R$, (ii)~$\concr{N'}$ has all non-negative activation functions,
    and (iii)~the new activation functions $\csigmap$ are of the form
    $\csigmap(x) = \max(\csigma(x) + \abs{C}, 0)$.
\end{restatable}
Notably, the proof of this theorem is \emph{constructive} with a
straightforward construction. The one requirement is that a lower-bound $C$ be
provided for the output of the nodes in the network. This lower-bound need not
be tight, and can be computed quickly using the same procedure discussed for
upper bounds immediately following Equation~1~in~\prabhakarcite{}.  For $\tanh$ in
particular, its output is always lower-bounded by $-1$ so we can immediately
take $C = -1$ for a network using only $\tanh$ activations.

\appendixproof{app:ShiftNetwork}

\subsection{Interval Hull Abstraction with Non-Continuous Functions}
\label{sec:ExampleINNThresh}
Another way that the constructed ANN may not over-approximate the DNN is if the
activation function does not satisfy the Weakened Intermediate Value
Property~(WIVP) (\pref{def:WIVP}). For example, consider the threshold
activation function (\pref{def:Activations}) with parameters $\mathtt{t} = 1$, $\mathtt{v} = 0$
and the same overall network, i.e.\
$f(x_1) =
{\scriptsize\begin{bmatrix}
    1 & 1 \\
    1 & 0 \\
    0 & 1
\end{bmatrix}}
\mathrm{thresh}\left(
    {\scriptsize\begin{bmatrix}
        1 \\
        -1
    \end{bmatrix}
    \begin{bmatrix}
        x_1
    \end{bmatrix}};
    1, 0
\right)$
and the same partitioning. We get the INN
$g(x_1) =
{\scriptsize\begin{bmatrix}
    [2, 2] \\
    [0, 2] \\
    [0, 2]
\end{bmatrix}}
\mathrm{thresh}\left(
    {\scriptsize\begin{bmatrix}
        [-1, 1]
    \end{bmatrix}
    \begin{bmatrix}
        x_1
    \end{bmatrix}};
    1, 0
\right)$.
We have $f(1) = \begin{bmatrix} 1 & 1 & 0 \end{bmatrix}^T$, however, in $g(1)$,
no matter how we instantiate the $[-1, 1]$ weight, the output of the
$\mathrm{thresh}$ unit will either be $0$ or $1$. But then the output of the
first output component must be either $0$ or $2$, neither of which is $1$, and
so $g$ does \emph{not} over-approximate $f$.

\subsubsection{Necessary Condition: WIVP}
\label{sec:NeedWIVP}

We can generalize this example to the following necessary condition:
\begin{restatable}{theorem}{ThmNeedWIVP}
    \label{thm:NeedWIVP}
    Suppose some $\csigma \in \concr{\Sigma}$ is an activation function which
    does not satisfy the WIVP, and every $\AbsDomain \in \AbsCollection$ is at
    least as precise as the interval hull abstraction.  Then there exists a
    neural network $\concr{N}$ using only the identity and $\csigma$ activation
    functions and partitioning $\nping$ such that $\abstr{T} = \ALW$ does not
    over-approximate $\concr{N}$.
\end{restatable}

\appendixproof{app:NeedWIVP}

While this is of some theoretical curiosity, in practice almost all
commonly-used activation functions do satisfy the WIVP. Nevertheless, if one
does wish to use such a function, one way to soundly over-approximate it with
an ANN is to replace the \emph{scalar} activation function with a
\emph{set-valued} one. The ANN semantics can be extended to allow picking any
output value from the activation function in addition to any weight from the
weight set.

For example, consider again the $\mathrm{thresh}(x; 1, 0)$ activation function.
It can be completed to a set-valued activation function which satisfies the
WIVP such as
${\scriptsize
\mathrm{thresh}'(x; 1, 0) \coloneqq
    \begin{cases}
        \{x\} &\text{if } x > 1 \\
        \{a\mid a \in [0, 1]\}  &\text{if } x = 1 \\
        \{0\} &\text{otherwise}
    \end{cases}
}$.
The idea is that we ``fill the gap'' in the graph. Whereas in the original
threshold function we had an issue because there was no $x \in [0, 1]$
which satisfied
$\mathrm{thresh}(x; 1, 0) = \frac{f(0) + f(1)}{2} = \frac{1}{2}$,
on the set-valued function we can take $x = 1 \in [0, 1]$ to find
$\frac{1}{2} \in \mathrm{thresh}'(1; 1, 0)$.

\subsection{Powerset Abstraction, ReLU, and $\haalpha_{bin}$}
\label{sec:ExamplePSNNReLU}
Recall that $\haalpha$~(\pref{alg:ahat}) requires abstracting the,
usually-infinite, set of \emph{all} merged matrices $D^T \cWi{i} C$.  However,
in~\pref{sec:AlgorithmComputable} we showed that for convex abstract domains it
suffices to only consider the finitely-many \emph{binary} mergings. The reader
may wonder if there are abstract domains for which it is \emph{not} sufficient
to consider only the binary PCMs. This section presents such an example.

Suppose we use the same ReLU DNN $f$ as in~\pref{sec:ExampleINNReLU}, for which
we noted before the corresponding INN over-approximates it.  However, suppose
instead of intervals we used the \emph{powerset} abstract domain, i.e.,
$\alpha(S) = S$ and $A \sqcup B = A \cup B$. If we (incorrectly) used
$\haalpha_{bin}$ instead of $\haalpha$, we would get the powerset ANN
$g(x_1) =
\left\{
    {\scriptsize\begin{bmatrix}
        2 \\
        2 \\
        0
    \end{bmatrix},
    \begin{bmatrix}
        2 \\
        0 \\
        2
    \end{bmatrix}}
\right\}
\mathrm{ReLU}\left(
    \{
        \begin{bmatrix} 1 \end{bmatrix},
        \begin{bmatrix} -1 \end{bmatrix}
    \}
    \begin{bmatrix}
        x_1
    \end{bmatrix}
\right)$.
Recall that $f(1) = \begin{bmatrix} 1 & 1 & 0 \end{bmatrix}^T$.  However, with
$g(1)$, the first output will always be either $0$ or $2$, so $g$ does
\emph{not} over-approximate $f$. The basic issue is that to get the
correct output, we need to instantiate the inner weight to $0.5$, which is in
the convex hull of the original weights, but is not either one of the original
weights itself.

Note that, in this particular example, it is possible to find an ANN that
over-approximates the DNN using only finite sets for the abstract weights.
However, this is only because ReLU is piecewise-linear, and the size of the
sets needed will grow exponentially with the number of dimensions. For other
activation functions, e.g., $\tanh$ infinite sets are required in general.

In general, non-convex abstract domains will need to use some other method of
computing an over-approximation of $\haalpha$. One general-purpose option is to
use techniques such as those developed for symbolic
abstraction~\cite{thakur_reps_CAV12} to iteratively compute an
over-approximation of the true $\AbsElemi{i}$ and use that instead.

\section{Proof of Sufficient Conditions}
\label{sec:ProveSound}

We now prove~\pref{thm:Sound}, which provides sufficient conditions on the
activation functions for which~\pref{alg:quotient} produces an ANN that soundly
over-approximates the given DNN.

The structure of the proof is illustrated in~\pref{fig:Soundness}. To show that
ANN $\abstr{T}$ over-approximates DNN $\concr{N}$, we must show that
$\concr{N}(\vec{v}) \in \abstr{T}(\vec{v})$ for every $\vec{v}$. This occurs,
by definition, only if there exists some \emph{instantiation}
$\concr{T_{\vec{v}}} \in \cgamma(\abstr{T})$ of $\abstr{T}$ for which
$\concr{N}(\vec{v}) = \concr{T_{\vec{v}}}(\vec{v})$. Recall that an
instantiation of an ANN is a DNN formed by replacing each abstract weight
$\AbsElemi{i}$ with a concrete weight matrix $\concr{H^{(i)}} \in
\cgamma(\AbsElemi{i})$. In particular, our proof will proceed layer-by-layer.
On an input $\vec{v} = \vec{v}^{(0)}$, the $i$th layer of DNN $\concr{N}$ maps
$\vec{v}^{(i-1)}$ to $\vec{v}^{(i)}$ until the output $\vec{v}^{(n)}$ is
computed.
We will prove that, for each abstract layer
$(\AbsElemi{i}, \csigmai{i}, \AbsDomaini{i})$,
there is a matrix
$\cHi{i} = \czeta(\AbsElemi{i}, \vec{v}^{(i-1)}) \in \cgamma(\AbsElemi{i})$
for which the instantiated layer
$(\cHi{i}, \csigmai{i})$, roughly speaking,
also maps $\vec{v}^{(i-1)}$ to $\vec{v}^{(i)}$.
However, by design the abstract layer will have fewer dimensions, hence the
higher-dimensional $\vec{v}^{(i-1)}$ and $\vec{v}^{(i)}$ may not belong to its
domain and range (respectively).
We resolve this by associating with each vector $\vec{v}^{(i)}$ in the
intermediate spaces of $\concr{N}$ a \emph{mean representative} vector
$\meanrep{\vec{v}^{(i)}}{\lping{i}}$ in the intermediate spaces of
$\concr{T_{\vec{v}}}$.
Then we can rigorously prove that the instantiated layer $(\cHi{i},
\csigmai{i})$ maps $\meanrep{\vec{v}^{(i-1)}}{\lping{i-1}}$ to
$\meanrep{\vec{v}^{(i)}}{\lping{i}}$.
Applying this fact inductively gives us
$\concr{T_{\vec{v}}}(\meanrep{\vec{v}}{\lping{0}}) =
\meanrep{(\concr{N}(\vec{v}))}{\lping{n}}$.  Because $\lping{0}$ and
$\lping{n}$ are the singleton partitionings, this gives us exactly the desired
relationship $\concr{T_{\vec{v}}}(\vec{v}) = \concr{N}(\vec{v})$.

\renewcommand\tikzwv[2]{
    \node (w#1) at (#2, 0) {$\vec{w}^{(#1)}$};
    \node (v#1) at (#2+1, 0) {$\vec{v}^{(#1)}$};
}
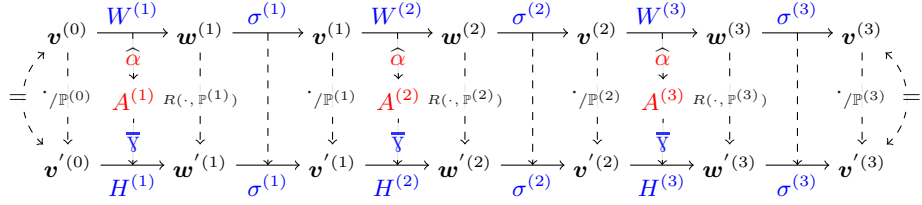
\begin{figure}
    \centering
\begin{tikzpicture}[scale=0.88]
    \node (v0) at (0, 0) {$\vec{v}^{(0)}$};
    \node (vp0) at (0, -2) {$\vec{v}^{'(0)}$};
    \foreach \i in {1,2,...,3}{
        \node (w\i) at (4*\i-2, 0) {$\vec{w}^{(\i)}$};
        \node (v\i) at (4*\i, 0) {$\vec{v}^{(\i)}$};
        \draw[->] (v\number\numexpr \i-1 \relax) -- (w\i) node[midway,above] (W\i) {$\cWi{\i}$};
        \draw[->] (w\i) -- (v\i) node[midway,above] (s\i) {$\csigmai{\i}$};

        \node (A\i) at (4*\i-3, -1) {$\AbsElemi{\i}$};

        \node (wp\i) at (4*\i-2, -2) {$\vec{w}^{'(\i)}$};
        \node (vp\i) at (4*\i,   -2) {$\vec{v}^{'(\i)}$};
        \draw[->] (vp\number\numexpr \i-1 \relax) -- (wp\i) node[midway,below] (H\i) {$\concr{H^{(\i)}}$};
        \draw[->] (wp\i) -- (vp\i) node[midway,below] (sp\i) {$\csigmai{\i}$};

        \draw[->] (A\i) -- (H\i)
        node[midway,fill=white,text opacity=1,opacity=0.9] {$\czeta$};
        \draw[->] (W\i) -- (A\i)
        node[midway,fill=white,text opacity=1,opacity=0.9] {$\haalpha$};
        \draw[->,dashed] (s\i) -- (sp\i);
    }
    \foreach \i in {0,1,...,3}{
        \draw[->,dashed] (v\i) -- (vp\i)
        node[midway,fill=white,text opacity=1,opacity=0.9] {$\meanrep{\cdot}{\lping{\i}}$};
    }
    \draw[<->,dashed] (v0) to [bend right=50] (vp0)
    node[fill=white,text opacity=1,opacity=0.9] at (-0.75, -1) {$=$};
    \draw[<->,dashed] (v3) to [bend left=50] (vp3)
    node[fill=white,text opacity=1,opacity=0.9] at (12.75, -1) {$=$};
    \foreach \i in {1,2,...,3}{
        \draw[->,dashed] (w\i) -- (wp\i)
        node[midway,fill=white,text opacity=1,opacity=0.9] {\tiny$R(\cdot, \lping{\i})$};
    }
\end{tikzpicture}
    \caption{
        Visualization of the relationships between concrete, abstract, and
        instantiated elements in the soundness proof. The original DNN's action
        on an input vector $\vec{v}^{(0)}$ is shown on the top row. This DNN is
        abstracted to an ANN, represented by the $\AbsElemi{i}$s on the middle
        row.  We will show that we can instantiate the ANN such that the
        instantiation has the same output as the original DNN on
        $\vec{v}^{(0)}$.
    }
    \label{fig:Soundness}
\end{figure}

\begin{figure}[t]
\begin{algorithm}[H]
    \DontPrintSemicolon
    \KwIn{
        An $n\times m$ matrix $M$. Partitionings $\ping^{in}$, $\ping^{out}$. A
        vector $\vec{v}$ with non-negative entries.  A vector $\vec{w}' \in
        R(M\vec{v}, \ping^{out})$.
    }
    \KwOut{
        A matrix $\concr{H} \in \cgamma(\haalpha(M, \ping^{in}, \ping^{out}))$
        such that $\concr{H}(\meanrep{\vec{v}}{\ping^{in}}) = \vec{w}'$.
    }
    $C, D \gets 0_{\abs{\ping^{in}}\times n}, 0_{\abs{\ping^{out}}\times m}$\;
    \For{$i = 1, 2, \ldots, \abs{\ping^{in}}$}{
        \For{$j \in \ping^{in}_i$}{
            $C_{j, i} \gets v_j / (\sum_{k \in \ping^{in}_i} v_k)$\;
        }
    }
    $\vec{w} \gets M\vec{v}$\;
    \For{$i = 1, 2, \ldots, \abs{\ping^{out}}$}{
        $a, b \gets \argmax_{p \in \ping^{out}_i} w_p, \argmin_{p \in \ping^{out}_i} w_p$\;
        $D_{a, i} \gets (w'_i - w_b) / (w_a - w_b)$\;
        $D_{b, i} \gets 1 - D_{a, i}$\;
    }
    $\vec{s} \gets (\abs{\ping^{in}_1}, \ldots, \abs{\ping^{in}_{\abs{\ping^{in}}}})$\;
    \returnKw{$\ScaleCols\left(D^T M C, \vec{s}\right)$}\;

    \caption{$\czeta(M, \ping^{in}, \ping^{out}, \vec{v}, \vec{w}')$}
    \label{alg:zeta}
\end{algorithm}
\end{figure}

\subsection{Vector Representatives}
\label{sec:VReps}
Our proof relies heavily on the concept of representatives.
\begin{definition}
    \label{def:CRep}
    Given a vector $\vec{v} = (v_1, v_2, \ldots, v_n)$ and a partitioning
    $\ping$ of~$\{1, 2, \ldots, n\}$ with $\abs{\ping} = k$, we define the
    \emph{convex representative set} of $\vec{v}$ under $\ping$ to be
    \\
    $
        R(\vec{v}, \ping) =
        \left\{
            (z_1, z_2, \ldots, z_k)
            \mid
            \forall j.
                \min_{h \in \ping_{j}} v_h
                \leq z_j \leq
                \max_{h \in \ping_{j}} v_h
        \right\}.
    $
\end{definition}
$R(\vec{v}, \ping)$ is referred to as $AV(\vec{v})$ in~\prabhakarcite{}, and is
always a box in $\mathbb{R}^k$.

One representative will be particularly useful, so we give it a specific
notation:
\begin{definition}
    \label{def:MRep}
    Given a vector $(v_1, v_2, \ldots, v_n)$ and a partitioning $\ping$ of $\{
        1, 2, \ldots, n \}$ with $\abs{\ping} = k$, we define the \emph{mean
    representative} of $v$ under $\ping$ to be
    \\
    $
        \meanrep{v}{\ping} =
        \left(
        \frac{\sum_{j \in \ping_{1}} v_j}{\abs{\ping_{1}}},
        \ldots,
        \frac{\sum_{j \in \ping_{k}} v_j}{\abs{\ping_{k}}}
        \right)
    $
\end{definition}

\begin{example}
    Consider the vector $\vec{v} \coloneqq (5, 6, 11, 2, 1)$ and the partitioning $\ping =
    \{\{1, 3\}, \{2, 4, 5\}\}$. Then we have
    $\meanrep{\vec{v}}{\ping} = ((5+11)/2, (6 + 2 + 1)/3) = (8, 3)$
    and
    $R(\vec{v}, \ping) = \{ (z_1, z_2) \mid z_1 \in [5, 11], z_2 \in [1, 6] \}$.
\end{example}

\subsection{Proof of Soundness Theorem}
The operation $\czeta$ presented in \pref{alg:zeta} shows how to instantiate an
abstract weight matrix such that it has input/output behavior corresponding to
that of the original DNN layer. We now prove the  correctness of \pref{alg:zeta}.

\begin{lemma}
    \label{lem:ZetaCorrect}
    Given any $\vec{w}' \in R(M\vec{v}, \ping^{in})$, a vector $\vec{v}$ with
    non-negative entries, and $\concr{H} = \czeta(M, \ping^{in}, \ping^{out},
    \vec{v}, \vec{w}')$, then $\concr{H} \in \cgamma(\haalpha(M, \ping^{in},
    \ping^{out}))$ and $\concr{H}(\meanrep{\vec{v}}{\ping^{in}}) = \vec{w}'$.
\end{lemma}
\begin{proof}
    To prove correctness of~\pref{alg:zeta}, it suffices to show that (i) $C$
    and $D$ are PCMs, and (ii) the returned matrix $\concr{H}$ satisfies the
    equality $\concr{H}(\meanrep{\vec{v}}{\ping^{in}}) = \vec{w}'$.

    $C$ is a PCM by construction: The $i$th column only has non-zero entries
    for rows that are in the $i$th partition. The sum of all entries in a
    column is
    $\sum_{j \in \ping^{in}_i} v_j / (\sum_{k \in \ping^{in}_i} v_k) = 1$.
    All entries are non-negative by assumption on $\vec{v}$.

    $D$ is also a PCM: The $i$th column only has two entries. It suffices to
    show that $D_{a, i}$ is in $[0, 1]$, which follows because $\vec{w}' \in
    R(M\vec{v}, \ping^{out})$ implies $w'_i$ is in between the minimum $b$ and
    maximum $a$.

    By associativity, line 11 is equivalent to returning $\concr{H} = D^TME$
    where $E = \ScaleCols(C, \vec{s})$.  Thus, to show that
    $\concr{H}(\meanrep{\vec{v}}{\ping^{in}}) = \vec{w}'$, it suffices to show
    (i) that $E(\meanrep{\vec{v}}{\ping^{in}}) = \vec{v}$, and (ii) that
    $D^TM\vec{v} = \vec{w}'$.

    Note that here $E_{j, i} = C_{j, i}\abs{\ping^{in}_i}$. Then to show (i),
    consider any index $j \in \ping^{in}_i$.  Then we find that the $j$th
    output component of $E(\meanrep{\vec{v}}{\ping^{in}})$ is\\
    $(v_j/(\sum_{k\in\ping^{in}_i v_k}))
    \abs{\ping^{in}_i}
    ((\sum_{k\in\ping^{in}_i v_k})/\abs{\ping^{in}_i}) = v_j$. Hence, the
    entire output vector is $\vec{v}$.

    To show (ii), note that each column of $D$ is exactly the convex
    combination that produces the output $w'_i$ from the maximum/minimum
    indices of $M\vec{v}$.

    In total then, the returned matrix is in $\cgamma(\haalpha(M, \ping^{in},
    \ping^{out}))$ and satisfies $\concr{H}(\meanrep{\vec{v}}{\ping^{in}}) = \vec{w}'$.
\end{proof}

The next lemma implies that we can always find such a $\vec{w}' \in R(M\vec{v},
\ping^{in})$ satisfying the relations in~\pref{fig:Soundness}.

\begin{lemma}
    \label{lem:PreimageOfMean}
    Let $\csigma$ be an activation function satisfying the WIVP, $\vec{w}$
    any vector, and $\ping$ a partitioning the dimensions of $\vec{w}$.
    Then there exists a vector \\ $\vec{w}' \in R(\vec{w}, \ping)$
    such that $\csigma(\vec{w}') = \meanrep{(\csigma(\vec{w}))}{\ping}$.
\end{lemma}

\begin{proof}
    Because $\csigmai{i}$ is defined to be a component-wise activation
    function, we can assume WLOG that $\lping{i}$ has only a single partition,
    i.e., $\lping{i} = \{\{1, 2, \ldots, s^{(i)}\}\}$.

    In that case, label the components of $\vec{w}^{(i)}$ such that $w_1^{(i)}
    \leq w_2^{(i)} \leq \ldots \leq w_n^{(i)}$. Then the statement of the lemma
    is equivalent to the assertion that there exists some $b \in [w_1^{(i)},
    w_n^{(i)}]$ such that $\csigmai{i}(b) = (\sum_j w_j^{(i)}) / n$. But this
    is exactly the definition of the WIVP. Hence, by assumption that
    $\csigmai{i}$ satisfies the WIVP, we complete the proof.
\end{proof}

We are finally prepared to prove the soundness theorem. It is restated here for
clarity.

\ThmSound*
\begin{proof}
    A diagram of the proof is provided in~\pref{fig:Soundness}.

    Consider the $i$th layer. By~\pref{lem:PreimageOfMean}, there exists some
    vector $\vec{w}^{'(i)} \in R(\vec{w}^{(i)}, \lping{i})$ such that
    $\csigmai{i}(\vec{w}^{'(i)}) = \meanrep{\vec{v}}{\lping{i}}$. Furthermore,
    by~\pref{lem:ZetaCorrect} there exists some $\concr{H^{(i)}} \in
    \cgamma(\AbsElemi{i})$ such that
    $\concr{H^{(i)}}(\meanrep{\vec{v}^{(i-1)}}{\lping{i-1}}) = \vec{w}^{'(i)}$.
    Therefore, in total we can instantiate the $i$th abstract layer to
    $(\concr{H^{(i)}}, \csigmai{i})$, which maps
    $\meanrep{\vec{v}^{(i-1)}}{\lping{i-1}}$ to
    $\meanrep{\vec{v}^{(i)}}{\lping{i}}$.

    By applying this construction to each layer, we find an instantiation of
    the ANN that maps $\meanrep{\vec{v}^{(0)}}{\lping{0}}$ to
    $\meanrep{\vec{v}^{(n)}}{\lping{n}}$. Assuming $\lping{0}$ and $\lping{n}$
    are the singleton partitionings, then, we have that the instantiation maps
    $\vec{v}^{(0)} = \vec{v}$ to $\vec{v}^{(n)} = N(\vec{v})$, as hoped for.
    Hence, $\concr{N}(\vec{v}) \in \abstr{T}(\vec{v})$ for any such vector
    $\vec{v}$, and so the ANN overapproximates the original DNN.
\end{proof}

\section{Related Work}
\label{sec:Related}

The recent results by~\prabhakarcite{} are the closest to this paper.
\prabhakar{} introduce the notion of Interval Neural Networks and a sound
quotienting (abstraction) procedure when the ReLU activation function is used.
Prabhakar~et~al.~also proposed a technique for verification of DNNs using ReLU
activation functions by analyzing the corresponding INN using a MILP encoding.
Prabhakar~et~al.~leaves open the question of determining the appropriate
partitioning of the nodes, and their results assume the use of the ReLU
activation function and interval domain. We have generalized their results to
address the subtleties of other abstract domains and activation functions as
highlighted in~\pref{sec:Examples}.

There exists prior work~\cite{beheshti1998interval,patino2004interval,856123}
on models using interval-weighted neural networks. The goal of such approaches
is generally to represent uncertainty, instead of improve analysis time of a
corresponding DNN.  Furthermore, their semantics are defined using interval
arithmetic instead of the more-precise semantics we give in~\pref{sec:ANNs}.
Nevertheless, we believe that future work may consider applications of our more
general ANN formulation and novel abstraction algorithm to the problem of
representing uncertainty.

There have been many recent approaches exploring formal verification of DNNs
using abstractions. ReluVal \cite{wang2018formal} computes interval bounds on
the outputs of a DNN for a given input range. Neurify \cite{wang2018efficient}
extends ReluVal by using symbolic interval analysis. Approaches such as
DeepPoly~\cite{singh2019abstract} and $\text{AI}^2$~\cite{ai2} perform abstract
interpretation of DNNs using more expressive numerical domains such as polyhedra
and zonotopes. In contrast, Abstract Neural Networks introduced in this paper
use abstract values to represent the weight matrices of a DNN, and are a
different way of applying abstraction to DNN analysis. 

This paper builds upon extensive literature on numerical abstract domains
\cite{DBLP:conf/popl/CousotC77,mine2017tutorial,DBLP:conf/popl/CousotH78,mine2006octagon},
including libraries such as APRON~\cite{jeannet2009apron} and
PPL~\cite{BagnaraHZ08SCP}. Of particular relevance are
techniques for verification of floating-point computation
\cite{chen2008sound,ponsini2016verifying,ponsini2016verifying}.

Techniques for compression of DNNs reduce their size using
heuristics~\cite{iandola2016squeezenet,deng2020model,han2015deep}. They can
degrade accuracy of the network, and do not provide theoretical guarantees.
Gokulanathan~et~al.~\cite{DBLP:journals/corr/abs-1910-12396} use the Marabou
Verification Engine~\cite{katz2019marabou} to simplify neural networks so that
the simplified network is equivalent to the given network.
Shriver~et~al.~\cite{DBLP:journals/corr/abs-1908-08026} refactor the given DNN
to aid verification, though the refactored DNN is not guaranteed to be an
overapproximation.

\section{Conclusion and Future Directions}
\label{sec:Conclusion}

We introduced the notion of an \emph{Abstract Neural Network (ANN)}. The weight
matrices in an ANN are represented using numerical abstract domains, such as
intervals, octagons, and polyhedra. We presented a framework, parameterized by
abstract domain and DNN activation function, that performs layer-wise
abstraction to compute an ANN given a DNN. We identified necessary and
sufficient conditions on the abstract domain and the activation function that
ensure that the computed ANN is a sound over-approximation of the given DNN.
Furthermore, we showed how the input DNN can be modified in order to soundly
abstract DNNs using rare activation functions that do not satisfy the
sufficiency conditions are used. Our framework is applicable to DNNs that use
activation functions such as ReLU, Leaky ReLU, and Hyperbolic Tangent.  Our
framework can use convex abstract domains such as intervals, octagons, and
polyhedra. Code implementing our framework can be found at 
\url{https://github.com/95616ARG/abstract_neural_networks}.

The results in this paper provide a strong theoretical foundation for further
research on abstraction of DNNs. One interesting direction worth exploring is
the notion of completeness of abstract
domains~\cite{DBLP:journals/jacm/GiacobazziRS00} in the context of Abstract
Neural Networks. Our framework is restricted to convex abstract domains; the use
of non-convex abstract domains, such as modulo
intervals~\cite{DBLP:conf/IEEEpact/NakanishiJPF99} or donut
domains~\cite{DBLP:conf/vmcai/GhorbalIBMG12}, would require a different
abstraction algorithm. Algorithms for computing symbolic abstraction might show
promise~\cite{DBLP:conf/vmcai/RepsSY04,li2014symbolic,bilateral,thakur_reps_CAV12,reps_thakur_VMCAI16}.

This paper focused on feed-forward neural networks. Because convolutional
neural networks~(CNNs) are special cases of feed-forward neural networks,
future work can directly extend the theory in this paper to CNN models as well.
Such future work would need to consider problems posed by non-componentwise
activation functions such as MaxPool, which do not fit nicely into the
framework presented here.  Furthermore, extensions for recursive neural
networks~(RNNs) and other more general neural-network architectures seems
feasible.

On the practical side of things, it would be worth investigating the impact of
abstracting DNNs on the verification times. \prabhakarcite{}
demonstrated that their abstraction technique improved verification of DNNs. The
results in this paper are a significant generalization of the results of
Prabhakar~et~al., which were restricted to interval abstractions and ReLU
activation functions. We believe that our approach would similarly help scale up
verification of DNNs.

\noindent\textbf{Acknowledgments}
We thank the anonymous reviewers and Cindy Rubio Gonz\'alez for their feedback
on this work.

\bibliographystyle{splncs04}
\bibliography{main}

\appendix
\onlyfor{arxiv}{\section{Proof of~\pref{thm:BinPCMs}}
\label{app:BinPCMs}

In~\pref{sec:AlgorithmComputable} we argued that, when the abstract domain used
is convex, $\haalpha$ and $\haalpha_{bin}$ produce the same result. Hence the
latter (which is computable) can be used in place of the former when
executing~\pref{alg:quotient}. We now prove this fact. To do so, we use the
following Lemma, which implies a similar claim for the PCMs:
\begin{lemma}
    \label{lem:EPCM}
    Let $\ping$ be a partitioning of $\{ 1, 2, \ldots, n \}$. Then every PCM $C
    \in \PCMs(\ping)$ is a convex combination of the binary PCMs
    $\BinPCMs(\ping)$.
\end{lemma}

\begin{proof}
    Let $C \in \PCMs(\ping)$. It suffices to construct $C$ as a convex
    combination of the binary PCMs $B_i \in \BinPCMs(\ping)$.

    We first note that the columns in the binary PCMs are independent --- given
    two binary PCMs $B_1$ and $B_2$ the PCM $B_3$ formed by setting the $i$th
    column of $B_3$ to be either the $i$th column of $B_1$ or that of $B_2$
    arbitrarily is always a valid binary PCM. Because of this, we can consider
    each column separately, i.e., assume that $C$ and the $B_i$s have only a
    single column.

    In that case, the binary PCMs can be thought of as length-$n$ vectors each
    with all zero entries except for a single $1$ entry. Similarly, $C$ must
    have zeros in the entries where the binary PCMs all have zeros, and the
    non-zero entries must be positive and sum to one. But this is equivalent to
    stating that $C$ lies in the convex span of the binary PCMs $B_i$, as
    claimed.
\end{proof}

Then we can prove
\ThmBinPCMs*

\begin{proof}
    By convexity, it suffices to show that every $D^T M C$ with $C$ and $D$
    being PCMs can be written as a convex combination of the $F^T M E$
    matrices, with $E$ and $F$ being binary PCMs of $\ping^{in}$ and
    $\ping^{out}$ respectively. Note that the scaling is uniform, so we may
    ignore the $\ScaleCols$.

    By~\pref{lem:EPCM} we can write $C$ as a convex combination of the binary
    PCMs $E_j$, i.e., $C = \sum_j \alpha_j E_j$. We can similarly write $D$ as
    a convex combination of the binary PCMs $F_k$, i.e., $D = \sum_k \beta_k
    F_k$.  By distributivity, then, we have
    \[
        {\scriptsize
        D^TMC = \sum_j \sum_k \alpha_j \beta_k F_k^T M E_j}.
    \]
    The $F_k^T M E_j$ matrices are exactly the binary mergings used by
    $\haalpha_{bin}$, hence it suffices now to show that the $\alpha_j\beta_k$
    coefficients are non-negative and sum to one. They are non-negative by
    construction, and sum to one because
    \[
        {\scriptsize
        \sum_j \sum_k \alpha_j \beta_k = \sum_j \left(\alpha_j \sum_k \beta_k\right)
        = \sum_j \alpha_j\cdot1 = \sum_j \alpha_j = 1}.
    \]
    Therefore, any such $D^T M C$ is a convex combination of the binary
    mergings, as claimed, and hence
    $\haalpha(M, \ping^{in}, \ping^{out}, \AbsDomain) =
    \haalpha_{bin}(M, \ping^{in}, \ping^{out}, \AbsDomain)$.
\end{proof}

\section{The WIVP is Strictly Weaker Than the IVP}
\label{app:WIVP}

In~\pref{sec:ExampleSoundness} we stated that the WIVP was strictly weaker than
the IVP, i.e., there exists a function $f$ which satisfies the WIVP but not the
IVP. The below proof constructs such a function.

\begin{proof}
    Let $g$ be any \emph{strongly Darboux} function, e.g., Conway's Base 13
    function~\cite{oman2014converse}. The strongly Darboux property implies
    that the postimage of any non-empty open interval under $g$ is all of
    $\mathbb{R}$.

    Then, let $h$ be any surjective map from $\mathbb{R} \to \mathbb{Q}$. For
    example, we can take
    \[
        h(x) =
        \begin{cases}
            x &x \in \mathbb{Q} \\
            0 &x \not\in \mathbb{Q}.
        \end{cases}
    \]

    Finally, define
    \[
        f(x) \coloneqq h(g(x))
    \]
    and note that (i) $f(x) \in \mathbb{Q}$ for every $x$ while (ii) the
    postimage of any non-empty open interval under $f$ is exactly $\mathbb{Q}$.

    First, $f$ does \emph{not} satisfy the IVP because, for instance, $1$ and
    $2$ are in the image of $f$ but not $\sqrt{2}$.

    On the other hand, $f$ \emph{does} satisfy the WIVP. To see this, consider
    any $a_1 \leq a_2 \leq \cdots \leq a_n$. If they are all the same, then
    $f(a_1)$ is equal to the average. Otherwise, note that each $f(a_i) \in
    \mathbb{Q}$ hence
    \[
        \frac{\sum_i f(a_i)}{n} \in \mathbb{Q} = f((a_1, a_n))
    \]
    and so there exists a $b \in (a_1, a_n) \subseteq [a_1, a_n]$ with $f(b) =
    \sum_i f(a_i) / n$ as desired.

    Therefore, the WIVP is strictly weaker than the IVP.
\end{proof}

\section{Proof of~\pref{thm:NeedNonNegative}}
\label{app:NeedNonNegative}
In~\pref{sec:NeedNonNegative} we stated a necessary condition for sound
abstraction via~\pref{alg:quotient}, namely that the activation functions have
non-negative outputs. We now proceed to prove this theorem, restated below.

\ThmNeedNonNegative*

\begin{proof}
    Label $x, y \in \mathbb{R}$ such that $\csigma(x) < 0$ but $\csigma(y) > 0$.
    We will then take the DNN defined by the function
    \[
        \concr{N}(v) =
        \begin{bmatrix}
            1 & 0 \\
            0 & 1 \\
        \end{bmatrix}
        \csigma\left(
        \begin{bmatrix}
            x \\
            y \\
        \end{bmatrix}
        v \right)
    \]
    and the partitioning
    \[
        P = (\{\{1\}\}, \{\{1, 2\}\}, \{\{1\}, \{2\}\})
    \]
    which collapses all of the hidden dimensions. Then the corresponding
    interval ANN is given by
    \[
        \abstr{T_{\IntDom}}(v) =
        \begin{bmatrix}
            [0, 2] \\
            [0, 2] \\
        \end{bmatrix}
        \csigma\left(
        \begin{bmatrix}
            [x, y]
        \end{bmatrix}
        v \right).
    \]
    But then the components of $\concr{N}(1)$ have opposite signs, which can
    never happen for an instantiation of $\abstr{T_{\IntDom}}(1)$, hence
    $\concr{N}(1) \not\in \abstr{T_{\IntDom}}(1)$.  If we used any more-precise
    abstraction than intervals to get an ANN $\abstr{T}$ we would have
    $\abstr{T}(1) \subseteq \abstr{T_{\IntDom}}(1)$, hence still $\concr{N}(1)
    \not\in \abstr{T}(1)$.  Therefore, $\abstr{T}$ does \emph{not}
    over-approximate $\concr{N}$, completing the proof.
\end{proof}

\section{Proof of~\pref{thm:ShiftNetwork}}
\label{app:ShiftNetwork}

\ThmShiftNetwork*

\begin{proof}
    We will construct $\concr{N'}$ with layers $(\cWip{i}, \csigmaip{i})$.

    We will first define the activation functions $\csigmaip{i}$ used by
    $\concr{N'}$. For every $i < n$ set $\csigmaip{i}(x) \coloneqq
    \mathrm{max}(\csigmai{i}(x) + \abs{C}, 0)$. For $i = n$ set $\csigmaip{n}
    \coloneqq \csigmai{n}$.
    Note that, by assumption for any input in $R$ we will have $\csigmai{i}(x)
    + \abs{C} > 0$, so the $\mathrm{max}$ is just needed to ensure that the
    activation function is formally non-negative on the entirety of
    $\mathbb{R}$.

    We now define the weight matrices $\cWip{i}$ used by $\concr{N}$. For $i =
    1$, set $\cWip{i} = \cWi{i}$. For every $i > 1$, define $\cWip{i}$ such
    that
    \[
        \cWip{i}\vec{v} \coloneqq \cWi{i}(\vec{v} - \abs{C}) = \cWi{i}\vec{v} - \cWi{i}\abs{C},
    \]
    where we have abused notation to let $C$ here refer to the vector with
    every component fixed to the constant $C$.  Note that adding constant terms
    such as this can be done by adding a single additional dimension to each
    layer according to a standard transformation.

    We now argue that $\concr{N}$ and $\concr{N'}$ have the same output on any
    vector $\vec{v}$. Let $\vec{v}^{(1)}$, $\vec{v}^{(2)}$, \ldots,
    $\vec{v}^{(n)}$ be the post-activation vector after each layer in the
    original DNN as defined in~\pref{def:DNN}, and $\vec{v}^{(1)'}$, \ldots,
    $\vec{v}^{(n)'}$ be the same for the constructed DNN on an input $\vec{v}$.

    Then we have, by construction,
    $
        \vec{v}^{(1)'} = \vec{v}^{(1)}, \vec{v}^{(2)'} = \vec{v}^{(2)} + \abs{C},
    $
    and for all $2 < i < n$ we have inductively
    \[
        \begin{aligned}
            \vec{v}^{(i)'}
            &= \csigmaip{i}(\cWip{i}\vec{v}^{(i-1)'}) \\
            &= \max(\csigmai{i}(\cWi{i}(\vec{v}^{(i-1)'} - \abs{C})) + \abs{C}, 0) \\
            &= \max(\csigmai{i}(\cWi{i}((\vec{v}^{(i-1)} + \abs{C}) - \abs{C})) + \abs{C}, 0) \\
            &= \max(\csigmai{i}(\cWi{i}\vec{v}^{(i-1)}) + \abs{C}, 0) \\
            &= \max(\vec{v}^{(i)} + \abs{C}, 0).
        \end{aligned}
    \]
    If $\vec{v} \in R$, then we have by assumption that $C$ is a lower-bound
    for the value of $\vec{v}^{(i)}$ and hence this gives simply
    \[
        \vec{v}^{(i)'} = \max(\vec{v}^{(i)} + \abs{C}, 0) = \vec{v}^{(i)} + \abs{C}.
    \]
    Finally, for the last layer we have
    \[
        \begin{aligned}
            \vec{v}^{(n)'}
            &= \csigmaip{n}(\cWip{n}\vec{v}^{(n-1)'}) \\
            &= \csigmai{n}(\cWi{n}(\vec{v}^{(n-1)'} - \abs{C})) \\
            &= \csigmai{n}(\cWi{n}((\vec{v}^{(n-1)} + \abs{C}) - \abs{C})) \\
            &= \csigmai{n}(\cWi{n}\vec{v}^{(n-1)}) \\
            &= \vec{v}^{(n)}
        \end{aligned}
    \]
    as desired.
\end{proof}

We now provide an example of this construction.

\begin{example}
    In the LeakyReLU example from~\pref{sec:ExampleINNLeaky}, suppose we are
    only interested in the behavior of $f$ for $x_1 \in [-1, 1]$. On that
    domain, the output of the LReLU is \emph{at least} $-0.5$, hence we can
    take $C = -0.5$. Applying the construction from the theorem, we have
$f'(x_1) =
{\scriptsize\begin{bmatrix}
    1 & 1 & -1 \\
    1 & 0 & -0.5 \\
    0 & 1 & -0.5
\end{bmatrix}}
\mathrm{LReLU}'\left(
    {\scriptsize\begin{bmatrix}
        1 & 0 \\
        -1 & 0 \\
        0 & 0.5
    \end{bmatrix}
    \begin{bmatrix}
        x_1 \\
        1
    \end{bmatrix}};
    0.5
\right)$
where \\
$\mathrm{LReLU}'(x; 0.5) \coloneqq \mathrm{max}(\mathrm{LReLU}(x; 0.5) + 0.5, 0)$.

Now consider, for example, $x_1 = -1$. In the original network we had
\[
    f(-1) = {\scriptsize\begin{bmatrix}
        1 & 1 \\
        1 & 0 \\
        0 & 1
    \end{bmatrix}}
    \mathrm{LReLU}\left(
        {\scriptsize\begin{bmatrix}
            1 \\
            -1
        \end{bmatrix}
        \begin{bmatrix}
            -1
        \end{bmatrix}};
        0.5
    \right) =
    {\scriptsize\begin{bmatrix}
        1 & 1 \\
        1 & 0 \\
        0 & 1
    \end{bmatrix}
    \begin{bmatrix}
        -0.5 \\
        1
    \end{bmatrix}}
    =
    {\scriptsize
    \begin{bmatrix}
        0.5 \\
        -0.5 \\
        1
    \end{bmatrix}},
\]
noticing that the output of the LReLU had a negative component. In the new
network, on the other hand, we have
\[
    f'(x_1) =
    {\scriptsize
    \begin{bmatrix}
        1 & 1 & -1 \\
        1 & 0 & -0.5 \\
        0 & 1 & -0.5
    \end{bmatrix}}
    \mathrm{LReLU}'\left(
    {\scriptsize
        \begin{bmatrix}
            1 & 0 \\
            -1 & 0 \\
            0 & 0.5
        \end{bmatrix}
        \begin{bmatrix}
            -1 \\
            1
        \end{bmatrix}};
        0.5
    \right)
    =
    {\scriptsize\begin{bmatrix}
        1 & 1 & -1 \\
        1 & 0 & -0.5 \\
        0 & 1 & -0.5
    \end{bmatrix}
    \begin{bmatrix}
        0 \\
        1.5 \\
        1
    \end{bmatrix}}
    =
    {\scriptsize\begin{bmatrix}
        0.5 \\
        -0.5 \\
        1
    \end{bmatrix}} = f(1)
\]
where we can see that indeed the output of the new activation function
$\mathrm{LReLU}'$ is non-negative.
\end{example}

\section{Proof of~\pref{thm:NeedWIVP}}
\label{app:NeedWIVP}
In~\pref{sec:NeedWIVP} we claimed that activation functions satisfying the WIVP
is a necessary condition for soundness of~\pref{alg:quotient}. This is
formalized by the theorem below, which we now prove.

\ThmNeedWIVP*

\begin{proof}
    Let $a_1 \leq a_2 \leq \cdots \leq a_n$ be the points violating the WIVP
    for $\csigma$, i.e.\ there does \emph{not} exist any $b \in [a_1, a_n]$
    such that
    \[
        \csigma(b) = \frac{\sum_i a_i}{n}.
    \]

    We can then take the DNN given by
    \[
        \concr{N}(v) =
        \begin{bmatrix}
            1 & 1 & \cdots & 1 \\
        \end{bmatrix}
        \csigma\left(
        \begin{bmatrix}
            a_1 \\
            a_2 \\
            \vdots \\
            a_n \\
        \end{bmatrix}
        v \right)
    \]
    and the partitioning
    \[
        P = (\{\{1\}\}, \{\{1, \ldots, n\}\}, \{\{1\}\})
    \]
    which collapses all of the hidden dimensions. Then in the interval
    abstraction we get the ANN
    \[
        \abstr{T}(v) =
        \begin{bmatrix}
            [n, n] \\
        \end{bmatrix}
        \csigma\left(
        \begin{bmatrix}
            [a_1, a_n]
        \end{bmatrix}
        v \right).
    \]
    Now, consider $\concr{N}(1)$ and $\abstr{T}(1)$. We have by definition that
    \[
        \concr{N}(1) = \csigma(a_1) + \cdots + \csigma(a_n).
    \]
    Suppose for sake of contradiction that $\abstr{T}$ over-approximates
    $\concr{N}$. Then we must have $\concr{N}(1) \in \abstr{T}(1)$. Then there
    must be an assignment to the weights in $\abstr{T}$ which matches
    $\concr{N}$, i.e., then there must be a $b \in [a_1, a_n]$ and $m \in [n,
    n]$ such that
    \[
        m\csigma(b) = f(1) = \sum_i \csigma(a_i).
    \]
    But $m \in [n, n]$ implies $m = n$, hence in that case we have
    \[
        \csigma(b) = \frac{\sum_i \csigma(a_i)}{n}.
    \]
    But $b \in [a_1, a_n]$, contradicting the assumption that $\csigma$
    violates the WIVP at $a_1, \ldots, a_n$ and so completing the proof.
\end{proof}
}{}

\end{document}